\RequirePackage{letltxmacro}
\LetLtxMacro{\LaTeXtextbf}{\textbf}
\LetLtxMacro{\LaTeXyear}{\year}
\documentclass{ieeeaccess_preprint}
\makeatletter
\gdef\@tfootnoteextra{This is an author-prepared preprint of the article published in \emph{IEEE Access}, vol.~14, pp.~93670--93693, 2026, doi: 10.1109/ACCESS.2026.3705574. IEEE Xplore: \url{https://ieeexplore.ieee.org/document/11571780}.
The published article is licensed under CC BY 4.0.}
\makeatother
\LetLtxMacro{\textbf}{\LaTeXtextbf}
\LetLtxMacro{\year}{\LaTeXyear}
\usepackage[english]{babel}
\usepackage{amsmath, amssymb}
\usepackage[linesnumbered, ruled, vlined, noend]{algorithm2e}
\usepackage{graphicx}
\usepackage{makecell}
\usepackage{multirow}
\usepackage{booktabs}
\usepackage{tikz}
\NewSpotColorSpace{PANTONE}
\AddSpotColor{PANTONE} {PANTONE3015C} {PANTONE\SpotSpace 3015\SpotSpace C} {1 0.3 0 0.2}
\SetPageColorSpace{PANTONE}%
\usepackage{pgfplots}
\pgfplotsset{compat=1.18}
\usetikzlibrary{arrows.meta, calc, decorations.pathreplacing}
\usepackage{cite}
\usepackage{amsthm}
\usepackage{xcolor}
\usepackage[T1]{fontenc}
\usepackage[breaklinks]{hyperref}
\usepackage{xurl}
\usepackage{array}
\usepackage{bm}
\makeatletter
\AtBeginDocument{\DeclareMathVersion{bold}
\SetSymbolFont{operators}{bold}{T1}{times}{b}{n}
\SetSymbolFont{NewLetters}{bold}{T1}{times}{b}{it}
\SetMathAlphabet{\mathrm}{bold}{T1}{times}{b}{n}
\SetMathAlphabet{\mathit}{bold}{T1}{times}{b}{it}
\SetMathAlphabet{\mathbf}{bold}{T1}{times}{b}{n}
\SetMathAlphabet{\mathtt}{bold}{OT1}{pcr}{b}{n}
\SetSymbolFont{symbols}{bold}{OMS}{cmsy}{b}{n}
\renewcommand\boldmath{\@nomath\boldmath\mathversion{bold}}}
\makeatother
\newtheorem{lemma}{Lemma}
\newtheorem{proposition}{Proposition}
\newcommand{\param}[1]{\texttt{#1}}
\newcommand{\dataset}[1]{\text{#1}}
\def\BibTeX{{\rm B\kern-.05em{\sc i\kern-.025em b}\kern-.08em
    T\kern-.1667em\lower.7ex\hbox{E}\kern-.125emX}}
\begin{document}
\history{Date of publication xxxx 00, 0000, date of current version xxxx 00, 0000.}
\doi{00.0000/ACCESS.0000.DOI}
\title{How Many Trees in a Random Forest? A Revisited Approach With Plateau Search and Optuna Integration}
\author{
\uppercase{Vadim A. Porvatov}\authorrefmark{1}, 
\uppercase{Andrey A. Dukhovny}\authorrefmark{1}, 
\uppercase{Andrey M. Lange\authorrefmark{2, 3}}} 
\address[1]{Sberbank, Moscow 117997, Russia (e-mail: vandporvatov@sberbank.ru, aadukhovny@sberbank.ru)}
\address[2]{Skolkovo Institute of Science and Technology (Skoltech), Moscow 121205, Russia (e-mail: a.lange@skoltech.ru)}
\address[3]{Federal Research Center ``Computer Science and Control" of Russian Academy of Sciences (FRC CSC RAS), Moscow 119333, Russia (e-mail: alange@frccsc.ru)}
\tfootnote{The work was supported by the grant for research centers in the field of AI provided by the Ministry of Economic Development of the Russian Federation in accordance with the agreement 000000C313925P4F0002 and the agreement with Skoltech No.~139-10-2025-033.}
\markboth
{V.A.~Porvatov et al.: How Many Trees in a Random Forest? A Revisited Approach}
{V.A.~Porvatov et al.: How Many Trees in a Random Forest? A Revisited Approach With Plateau Search}
\corresp{Corresponding author: Andrey Lange (a.lange@skoltech.ru)}
\begin{abstract}
Hyperparameter optimization (HPO) for Random Forest faces a specific difficulty in tuning the number of trees: the predictive score typically improves monotonically with ensemble size, so standard methods such as Tree-structured Parzen Estimator (TPE) and Hyperband require a predefined search range and often drive the estimate toward its right boundary. Early-stopping strategies avoid fixing such a range, but can be sensitive to score noise and prone to premature stopping.
To address this, we propose an integrated triplet-based plateau-search algorithm that removes the number of trees from the direct TPE search space
and still exploits information accumulated across HPO trials. 
The method adaptively tracks a near-minimal sufficient ensemble size by monitoring relative changes in the out-of-bag (OOB) score across a triplet of forest sizes and shifting this triplet accordingly. This yields an automated and user-interpretable procedure based on a tolerance parameter.
We also provide a theoretical analysis: we relate the proposed relative OOB-score criterion to the gap between the current and limiting scores, and derive an asymptotic variance estimate for the corresponding OOB-based absolute relative difference. 
Experiments show that the selected number of trees can differ substantially from the common heuristic: for most classical benchmark datasets it is smaller, whereas for some high-dimensional bioinformatics datasets, such as Arcene and Dorothea, it is larger.
The source code and reproducible experiments are available at
\url{https://github.com/lange-am/rf\_plateau\_hpo}.
\end{abstract}
\begin{keywords}
Machine Learning, Tree Ensembles, Random Forest, Ensemble Size, Number of Trees, Hyperparameter Search, Sequential Model-Based Optimization
\end{keywords}
\titlepgskip=-21pt
\maketitle
\section{Introduction}
\subsection{Random Forests: Overview and Advantages}
One of the most powerful machine learning models for both classification and regression tasks remains the 
Random Forest (Breiman~\cite{Breiman1996Bagging}, \cite{Breiman2001Random}), Breiman and Cutler~\cite{Breiman2003Manual},
Biau and Scornet~\cite{Biau2016Random}). 
It is an ensemble composed of classical decision trees, each trained independently from the others.
Alongside Gradient Boosting models 
(Schapire~\cite{Schapire1999Brief}, Friedman~\cite{Friedman2001Greedy, Friedman2002Stochastic}), 
tree-based ensembles continue to be state-of-the-art in predictive performance on tabular data, 
often rivaling deep learning algorithms. 
This has been noted in several recent studies: 
Borisov et al.~\cite{Borisov2024Deep}, 
Grinsztajn et al.~\cite{Grinsztajn2022Why}, 
and Shwartz-Ziv and Armon~\cite{shwartz2022tabular}.

Boosting algorithms are considered more powerful in terms of predictive accuracy than Random Forest, leading to a wider array of well-known implementations such as 
XGBoost (Chen and Guestrin~\cite{Chen2016Xgboost}), 
LightGBM (Ke et al.~\cite{Ke2017Lightgbm}), and 
CatBoost (Prokhorenkova et al.~\cite{Prokhorenkova2018Catboost}).
However, it was Random Forest that emerged as the leading predictive model in the large benchmark study by Fern{\'a}ndez-Delgado et al.~\cite{Fernandez-Delgado2014Do}, which employed 179 classifiers and 121 UCI datasets (Kelly et al.~\cite{Kelly2023Uci}).
Random Forest is implemented in the Python scikit-learn library 
(Pedregosa et al.~\cite{Pedregosa2011scikit-learn}) and in many R packages, including 
randomForest (Liaw and Wiener~\cite{Liaw2002Classification}), 
party (Hothorn et al.~\cite{Hothorn2006Unbiased}),
partykit (Hothorn and Zeileis~\cite{Hothorn2015partykit}),
ranger (Wright and Ziegler~\cite{Wright2017ranger}),
tuneRanger (Probst~\cite{probst2018tuneranger}; Probst et al.~\cite{Probst2019Hyperparameters}),
and the recently introduced optRF (Lange et al.~\cite{Lange2025optRF}).

The fundamental difference between boosting and forest lies in their learning mechanics and how they minimize a given loss function $L$. 
Let the full training set be denoted by 
$\mathcal{D} = \{(x_i, y_i)\}_{i=1}^n$.
For each tree $t$, a bootstrap sample $\mathcal{D}_t$ is drawn from $\mathcal{D}$ with replacement. 
The decision tree $h_t(x)$ is trained on $\mathcal{D}_t$ to minimize a loss function $L$ (e.g., log-loss or cross-entropy for classification or mean squared error for regression):
\begin{equation*}
h_t \approx \arg\min_{h} \sum_{(x_i, y_i) \in \mathcal{D}_t} L\big(y_i, h(x_i)\big).
\end{equation*}
The final prediction is the average of all trees:
\begin{equation}
\hat{y}_{\text{RF}}(x) = \frac{1}{T} \sum_{t=1}^{T} h_t(x).
\label{eq:rf_def}
\end{equation}

In contrast, Gradient Boosting builds its model sequentially. 
At each step $t$, a tree $h_t(x)$ is built to approximate the negative gradient of a ``differentiable'' loss $L$ with respect to the current ensemble $\hat{y}_{t-1}(x)$.
It solves:
\begin{equation*}
h_t \approx \arg\min_{h} \sum_{(x_i, y_i) \in \mathcal{D}_t} L\!\left( \; - \left. \frac{\partial L(y_i, F)}{\partial F} \right|_{F=\hat{y}_{t-1}(x_i)} ,\; h(x_i) \right).
\end{equation*}
For example, for a quadratic loss $L(y, F) = (y - F)^2$, this reduces to fitting the raw residuals $-L'_{F}(y_i, \hat{y}_{t-1}(x_i)) = 2(y_i - \hat{y}_{t-1}(x_i))$. 
The tree is then added to the ensemble with a learning rate $\eta$:
\begin{equation*}
\hat{y}_{\text{GB}}(x) = \hat{y}_{0}(x) + \eta \sum_{t=1}^{T} h_t(x).
\end{equation*}

Thus, both methods iteratively train trees $h_t(x)$ that minimize a loss $L$, but they differ in the first argument of $L$: Random Forest uses the original targets $y_i$, whereas boosting uses the negative gradient of the loss from the previous step.
In other words, in Gradient Boosting, each new tree in the ensemble aims to correct the errors of all previous trees (i.e., the existing ensemble), 
whereas in Random Forest, 
all $T$ trees are trained independently. 

This allows boosting to advance further in reducing the approximation error of the target variable. 
However, the Random Forest solution exhibits greater stability, which stems from the Law of Large Numbers, as its prediction \eqref{eq:rf_def} is an average over a large number of independent (conditional on data) decision trees. 
From \eqref{eq:rf_def} it follows that the conditional variance decreases as 
\begin{equation}
\operatorname{Var} \left[\hat{y}_{\text{RF}}(x) \mid D \right] = O \!\left(\frac{1}{T}\right), \quad T \to \infty,
\label{eq:var}
\end{equation}
and, consequently, the conditional standard deviation is of order $O(T^{-1/2})$.
This stability is particularly valuable in high-dimensional settings where the number of features $p$ is large, 
especially when $p$ significantly exceeds the number of data samples $n$ 
(Diaz-Uriarte and De Andres~\cite{Diaz-Uriarte2006Gene}, 
Goldstein et al.~\cite{Goldstein2010Application, Goldstein2011Random},
Janitza et al.~\cite{Janitza2018Computationally}).

Furthermore, the independence of the constituent trees enables the use of the so-called out-of-bag (OOB) score when bootstrap sampling is applied.
The OOB score serves as an excellent alternative to cross-validated score estimation, 
thereby avoiding the need for repeated model training on different data folds and allowing the use of a single Random Forest training run to evaluate performance on the entire training dataset.
The independence of the trees also enables parallel training across multiple CPUs, greatly speeding up the forest training.

\subsection{Feature Importance Measures}
Another powerful feature of Random Forests is the availability of variable importance measures (global VIMs), which quantify each feature’s contribution either to the model’s predictions or to its predictive performance.
A canonical example is the Mean Decrease in Impurity (MDI; also known as the Gini importance), which naturally arises from the tree-building procedure itself (Breiman~\cite{Breiman1984Classification,Breiman2001Random}, Louppe et al.~\cite{Louppe2015Understanding}).
In the context of Random Forests, MDI is often viewed as more adequate under strong feature correlations than its direct analogue in Gradient Boosting, where trees are sequentially fitted and are not exchangeable.
The Random Forest Gini VIM can be interpreted as quantifying the associative strength between each predictor and the target variable: it captures non-linear effects, naturally handles mixed feature types (binary/categorical/continuous), and reflects interactions beyond simple univariate feature--target associations.

These importance scores are widely used in feature-selection pipelines.
For Minimal Feature Selection (finding the smallest subset of predictors sufficient for accurate prediction), a classical approach is Recursive Feature Elimination (RFE, Guyon et al.~\cite{Guyon2002Gene}), which iteratively removes features with the lowest importance.
For All-Relevant Feature Selection (identifying all predictors that carry signal about the target), algorithms such as Boruta (Kursa and Rudnicki~\cite{Kursa2010Feature}) and VITA (Janitza et al.~\cite{Janitza2018Computationally}) rely on Random-Forest-based VIMs.
Beyond feature selection, related importance notions are used in systems biology for network inference (Marbach et al.~\cite{Marbach2012Wisdom}), most notably in GENIE3 (Huynh-Thu et al.~\cite{Huynh-Thu2010Inferring}), and also in clustering and grouping of features (Tolosi and Lengauer~\cite{Tolosi2011Classification}).

From a methodological standpoint, the most common routes to global VIMs can be grouped as follows:
\begin{itemize}
    \item \textbf{Tree-ensemble intrinsic measures:} impurity-based importances such as MDI (Breiman~\cite{Breiman1984Classification,Breiman2001Random}) and related analyses (Louppe et al.~\cite{Louppe2015Understanding});
    \item \textbf{Permutation-based measures:} model-agnostic importance via feature shuffling (often referred to as MDA in Random Forests; Breiman~\cite{Breiman2001Random}), with well-documented sensitivity to correlations and other biases (Strobl et al.~\cite{Strobl2007Bias}; Hooker et al.~\cite{Hooker2021Unrestricted});
    \item \textbf{Shapley-value explanations:} SHAP (Lundberg and Lee~\cite{Lundberg2017Unified}; Covert et al.~\cite{Covert2020Understanding}), which is typically local but can be aggregated into global importances; for tree ensembles, TreeSHAP makes such computations practical at scale (Lundberg et al.~\cite{Lundberg2018Consistent,Lundberg2020From});
    \item \textbf{Neural-network attribution methods:} e.g., Integrated Gradients (Sundararajan et al.~\cite{Sundararajan2017Axiomatic}) and DeepLIFT (Shrikumar et al.~\cite{Shrikumar2017Learning}).
\end{itemize}

Permutation and Shapley-value approaches are model-agnostic and can therefore be applied to Random Forests as well; however, in high-dimensional tabular settings, computational constraints (for SHAP) and correlation-induced artifacts (particularly for permutation importance) become prominent.
For neural network attributions, several works report reliability and stability problems, showing that explanations can change substantially under small perturbations or even fail basic sanity checks (Adebayo et al.~\cite{Adebayo2018Sanity}; Ghorbani et al.~\cite{Ghorbani2017Interpretation}; Sixt et al.~\cite{Sixt2018When}).
Overall, despite its classical nature, Random Forest remains a practical backbone for global VIMs on tabular data; nevertheless, feature correlations and algorithmic randomness make stability a central concern, motivating the careful choice of the ensemble size discussed next.

\subsection{Why the Proper Number of Trees Is Critical}
When the number of features $p$ is large, strong inter-feature correlations are common.
During the construction of a decision tree, the best feature is selected at each split.
Consider an extreme case where several features are virtually identical.
Then the choice among them becomes arbitrary and effectively random.
Although this randomness may have little impact on the predictive rule of the model, it directly affects which feature receives credit in the importance score.

The key parameter governing the stability of both the model's predictions and its variable-importance measures (VIMs) is the number $T$ of random trees in the ensemble.
Prediction stability is typically attained relatively quickly as $T$ increases---a phenomenon studied theoretically and empirically by Scornet~\cite{Scornet2016Asymtotics}, 
Biau and Scornet~\cite{Biau2016Random}, and 
Wager et al.~\cite{Wager2014Confidence}.
However, as the example above illustrates, stabilizing VIMs generally requires a substantially larger number of trees than stabilizing the predictive score itself (Lange et al.~\cite{Lange2025optRF}).

Therefore, determining a sufficient number of trees that ensures both strong predictive performance and reliable feature-importance estimation is an important and timely problem.
Crucially, good predictive performance is \emph{necessary} but \emph{not sufficient} for reliable VIMs: feature-importance analysis is of limited value if the underlying model is poorly tuned.
It is precisely this necessity that motivates the present work.
Both score and VIM stabilities are naturally addressed as two sequential stages.
First, one tunes the Random Forest---including selecting $T$ together with other hyperparameters---to maximize the OOB score, thereby securing predictive quality and generalization.
Afterwards, if required, one can further increase the number of trees (keeping the remaining hyperparameters fixed) until the VIM stabilizes.

In this work, we focus on the first, necessary stage: how to identify a value of $T$ that is not excessively large---so as to avoid wasting computational resources on tuning oversized ensembles---yet is sufficient for high-quality predictions.

\subsection{Challenges, limitations of Current Practices and Proposed Approach}
However, there are two important considerations. 
First, in current practice, hyperparameter tuning requires explicit specification of search ranges, regardless of the search method --- be it grid search or any variant of random search, 
such as Sequential Model-Based Optimization via the Tree-structured Parzen Estimator (TPE, Bergstra et al.~\cite{Bergstra2011Algorithms}). 
Typically, the search range for $T$ is defined as an interval $[T_{\text{min}}, T_{\text{max}}]$, where, for example, one might choose $1 \leq T_{\text{min}} \leq 100$ and $500 \leq T_{\text{max}} \leq 5000$.
When this range is defined, the value identified by HPO tends to converge toward the upper bound $T_{\max}$, because in Random Forests (and unlike Gradient Boosting) increasing the number of trees does not cause overfitting and generally improves performance. 

It should be noted, however, that the model score plateaus as $T$ increases 
(see Fig.~\ref{fig:plateau}). 
Consequently, raising the $T_{\max}$ shifts the estimate of $T$ toward that boundary, albeit at a diminishing rate. 
The fact that it does not exactly reach the bound is due to inherent randomness --- either from the stochastic search procedure itself or, in the case of deterministic grid search, from the randomness in the Random Forest algorithm and the finite data sample. 
Crucially, regardless of the chosen bound, we lack \emph{certainty} or guarantees that it is sufficiently high, nor that it is not excessive.

Current practice sometimes suggests, for example, doubling the number of trees
as long as the score difference remains significant (Oshiro et al.~\cite{Oshiro2012How}).
Although in such a stopping rule $T_{\max}$ may effectively serve only as a safety cap and can therefore be chosen sufficiently large, the procedure remains vulnerable to false early stopping. In particular, a monotone early-stopping scheme tests the stopping condition sequentially from smaller to larger ensembles, so an early random satisfaction of the criterion may bias the selected tree count downward. 
Moreover, reliably distinguishing statistically significant score differences from negligible ones would in principle require training multiple Random Forests with $T$ and $2T$ trees, which is computationally expensive and also makes the procedure dependent on the chosen statistical test and its tuning parameters. 
Although increasing the number of such replicated forests can make the decision more stable, it does not fully remove the tendency toward premature stopping and underestimation of the required tree count.

In this work, we instead define sufficiency of $T$ through the relative score difference between models with $T$ and $T\cdot \textnormal{sf}$ trees, where the scale factor $\textnormal{sf}>1$ (e.g., $\textnormal{sf}=2$). 
This embeds the plateau principle directly into the criterion itself, rather than using it only as a heuristic guideline. Instead of repeatedly training many Random Forests for each candidate value of $T$, we exploit the variability naturally accumulated across HPO trials. 
This reduces the per-trial computational burden and mitigates the sensitivity
and downward-bias problems associated with simple monotone early-stopping schemes.

Another important point is that the search for a sufficient $T$ is \emph{interdependent} with other hyperparameters, such as tree depth, the number of randomly selected candidate features at each split (often called \param{mtry} or \param{max\_features}), the minimum number of samples in a leaf node, the minimum number of samples required to split a node, and others.
Therefore, simply doubling $T$ is not entirely appropriate when other hyperparameters are fixed, especially if $T$ is still small. 
For instance, if $T$ is set too low, or if the lower bound of the search range is too restrictive, the optimal estimated tree depth may increase to compensate for the limited approximation capacity of the ensemble. 

This appears somewhat paradoxical: reducing $T$ limits the Random Forest’s ability to reduce variance and mitigate overfitting, whereas increasing tree depth also increases variance and provides no compensation.
Hence, the optimum is not found in individual hyperparameters but in the entire set jointly. 
Despite sustained attention to selecting the number of trees and recognition of the problem caused by the absence of an interior optimum, insufficient attention has been paid to the interdependence among all hyperparameters.
Our proposed solution addresses this aspect.
We experimentally demonstrate that tuning hyperparameters jointly yields a significant improvement in the predictive score.

Our algorithm finds the optimal hyperparameter set, including $T$, in a single iterative TPE process using the Optuna framework (Akiba et al.~\cite{Akiba2019Optuna}). 
$T$ is excluded from the set of hyperparameters sampled in a Bayesian manner within predefined ranges. 
Instead, the number of trees is adapted based on the degree of plateau attainment. 
To assess this, we introduce a triplet of $T$ values differing by a specified scale factor and compute the relative error between successive OOB scores. 
Depending on this error, the working number of trees can be adjusted upward or downward at each trial of the Optuna TPE procedure. 
This yields an estimate of $T$ with a prescribed tolerance, e.g., $\varepsilon = 10^{-3}$. 
Unlike choosing an arbitrary upper bound $T_\text{max}$, selecting this tolerance parameter is intuitive and user-friendly.
\subsection{Overview of Contributions and Paper Structure}
This work is structured as follows. 
Section~\ref{sec:related_work} reviews existing approaches to varying and estimating the number of trees in Random Forest as well as algorithms for hyperparameter tuning. 
Section~\ref{sec:algorithm} introduces our core contribution: the plateau-search algorithm with its triplet-based adaptive mechanism, integrated within the Optuna TPE sampler.
It also provides a theoretical analysis of the proposed criterion: it establishes a link between the relative OOB-score plateau condition and the gap to the limiting score, derives asymptotic expressions for the conditional variance of both the signed and absolute plateau statistics, discusses key features of the algorithm, and addresses practical selection of the tolerance parameter, including its natural quantization scale for common classification metrics.
The implementation details of the accompanying Python library are also provided. 
Section~\ref{sec:results} presents a comprehensive experimental evaluation on benchmark datasets, comparing the proposed method against classical TPE- and Hyperband-based tuning and analyzing the impact of key parameters. 
Finally, Section~\ref{sec:conclusion} summarizes our findings, discusses the implications for efficient Random Forest tuning, and outlines directions for future work.
\section{Related work}
\label{sec:related_work}
\subsection{Trying different numbers of trees}
\begin{table*}[t]
\centering
\caption{Representative works related to choosing the number of trees $T$ in Random Forests.}
\label{tab:trees_related}
\setlength{\tabcolsep}{0.7cm}
\begin{tabular}{lll}
\toprule
\textbf{Study} &
\textbf{Result type} &
\textbf{Compared settings} \\
\midrule
Genuer et al.\ (2008) &
Empirical guidance &
Score on a fixed grid of $T$ \\
Oshiro et al.\ (2012) &
Empirical guidance &
Score on a geometric grid of $T$ \\
Latinne et al.\ (2001) &
Stopping criterion &
Score at $T$ vs.\ score at $T'>T$ \\
Hern\'andez-Lobato et al.\ (2013) &
Theoretical estimate &
Predictions at $T$ vs.\ $T'=\infty$ \\
Lange et al.\ (2025) &
Stopping criterion &
Reproducibility stability at $T$ vs.\ $T'>T$ \\
Wager et al.\ (2014) &
Analytical note &
Prediction variance vs.\ $T$ \\
Lopes (2016) & 
Analytical note &
Score variance vs.\ $T$, upper bound  \\
Lopes (2019) &
Analytical note &
Prediction variance vs.\ $T$, score at $T$ vs.\ $T'=\infty$  \\
Arlot and Genuer (2014) &
Theoretical estimate &
Prediction bias at $T$ vs.\ $T'=\infty$ (PRF) \\
Probst and Boulesteix (2018) &
Analytical note &
Score vs.\ $T$ \\
Cuzzocrea et al.\ (2013) &
Heuristic mapping &
$T$ vs.\ dataset ``predictive power'' \\
Demidova and Ivkina (2020) &
Empirical estimate &
$[T_{\min},T_{\max}]$ vs.\ $n$ \\
\bottomrule
\end{tabular}
\end{table*}
The challenge of selecting the number of trees in Random Forest can be traced to its early documentation. 
Breiman and Cutler~\cite{Breiman2003Manual} note an empirical rule of thumb: 
``Sometimes, I run out to 5000 trees if there are many variables and I want the variable importances to be stable.''
This statement highlights the inherently empirical and computationally intensive nature of the approach at the time, 
where the primary strategy was to increase the number of trees to a ``sufficiently large'' value to ensure stability.
Subsequent studies, including those by 
Strobl and Zeileis~\cite{Strobl2008Danger}, 
Genuer et al.~\cite{Genuer2010Variable}, and 
Kursa and Rudnicki~\cite{Kursa2010Feature}, 
similarly advocated for using as many trees as computationally feasible to achieve reliable results.

At the same time, this approach offers no clear criterion for selecting the number of trees $T$: ``The number of trees in a forest is a parameter that is not tunable in the classical sense but should be set sufficiently high'' (Probst et al.~\cite{Probst2019Hyperparameters}).
The influence of the number of trees $T$ on Random Forest performance has been empirically explored in several studies, including those by Genuer et al.~\cite{Genuer2008Random}, Oshiro et al.~\cite{Oshiro2012How}, and Cuzzocrea et al.~\cite{Cuzzocrea2013Information} (Table \ref{tab:trees_related}).
These works confirm that the model's score plateaus as $T$ increases, meaning that beyond a certain point, further improvements in quality become negligible.
Oshiro et al.\ further report that the required number of trees grows with the number of features $p$.

Oshiro et al.'s work ``How Many Trees in a Random Forest?''  considers the number of trees as a power of two, $T = 2^j$ for $j = 1, 2, \dots, 12$, i.e., up to 4096 trees. 
The dependence of the required $T$ on $n$, $p$, and the number of classes in multiclass classification is justified through the Vapnik–Chervonenkis (VC) dimension. In their study, they introduce the notion of dataset density, defined as $D = \log_p n$, and empirically demonstrate that a larger number of trees is required for low-density datasets, where $n \ll p$. 
The lowest-density dataset in their experiments has
$n = 60$, $p = 7129$, and $D = 0.46$. 
This study motivates our choice of a geometric grid for the number of trees in the proposed algorithm.

Similarly, Genuer et al.~\cite{Genuer2008Random} extensively experiment with Random Forest parameters. 
They confirm the performance plateau for  $T$ (testing values like $500, 1000, \dots, 5000$) and crucially show that for high-dimensional datasets ($n \ll p$), increasing the number of candidate features per split (\param{mtry}) beyond the default value of $\sqrt{p}$ often yields better performance. 
This insight directly points to the interdependence of hyperparameters and the need for their joint optimization, which is the core problem addressed in our work.

Cuzzocrea et al.~\cite{Cuzzocrea2013Information} study empirically how increasing the number of trees affects accuracy across multiple datasets, and relate the observed gains (e.g., comparing 10 vs.\ 500 trees) to information-theoretic measures of dataset ``predictive power''.

The work of Probst and Boulesteix~\cite{Probst2018To} is a fundamental study that analytically demonstrates that the expected classification error rate can be a non-monotonic function of the number of trees in a Random Forest when the data contain observations $i$ for which the error probability $\varepsilon_i$ of a single tree is close to or exceeds $0.5$. 
In particular, the area under the Receiver Operating Characteristic curve (ROC-AUC) may exhibit non-monotonic behavior due to changes in the ranking order of observations as $T$ increases. 
This generally contradicts the principle ``the more trees, the better.''
In contrast, metrics such as the Brier score, logarithmic loss, and mean squared error are guaranteed to improve monotonically with increasing $T$. 
However, in our experiments on well-known classification datasets, such non-monotonic behavior of ROC-AUC was not observed.
\subsection{Obtaining the number of trees}
Methods for estimating the optimal number of estimators in an ensemble have been proposed, starting with the early work of Latinne et al.~\cite{Latinne2001Limiting}, followed by Hern{\'a}ndez-Lobato et al.~\cite{Hernandez2013How}, and most recently by Lange et al.~\cite{Lange2025optRF}.
We highlight several conceptual similarities and differences between these approaches.
Below we contrast them in terms of what is compared, what notion of stability is targeted, which statistical decision rule is used, and how practical design choices affect the resulting ranges of recommended tree counts.

In addition to these three representative approaches, several complementary lines of work are worth noting.
Wager et al.~\cite{Wager2014Confidence} and Lopes~\cite{Lopes2016Sharp, Lopes2019Estimating} study finite-ensemble (Monte Carlo) effects consistent with the $O(T^{-1/2})$ scaling in \eqref{eq:var}, and discuss how large $T$ should be to make this component negligible.
Arlot and Genuer~\cite{Arlot2014Analysis} considered \emph{purely random forests}, an abstract model in which the partitioning mechanism is independent of the dataset, and analyzed how the approximation bias of a finite forest approaches that of the infinite-forest limit. 
They showed that, as tree complexity increases (e.g., with the number of leaves or the depth), the bias of the infinite forest decreases faster than the bias of a single tree, and derived a sufficient ensemble size for a finite forest to attain the same bias rate as the infinite one.
In some of their models, 
this sufficient size also depends on the input dimension $p$.
Demidova and Ivkina~\cite{Demidova2020Approach} fit dataset-size-driven bounds for the tree count, yielding logarithmic-type rules of the form $T \approx a\log(n)+b$.
Cuzzocrea et al.~\cite{Cuzzocrea2013Information} further outline an information-theoretic heuristic for deriving $T$ from dataset characteristics (see above).

Latinne et al., Hern{\'a}ndez-Lobato et al., and Lange et al.\ share a common underlying idea: to find a number of trees $T$ such that the differences between models of size $T$ and a larger size $T' > T$ become negligible.
A specific technical similarity is that both Latinne et al.\ and Lange et al.\ use an additive increment of $T' - T = 10$ trees in their criterion.
However, the approaches differ in what is being compared. 
One can either compare predictive performance with respect to ground truth---e.g., via a paired test on prediction correctness, as in Latinne et al.\ --- or quantify the agreement (\emph{stability}) of model outputs, as in Hern\'andez-Lobato et al.\ and Lange et al.
In this sense, our proposed method aligns with the first type, 
as we compare the relative change of the OOB score.

Regarding the notion of stability, for clarity we use the following informal terminology to distinguish two notions of ``stability'':
\begin{enumerate}
    \item \textbf{Convergence Stability}: The difference between predictive models with $T$ and $T'$ trees (Latinne et al.; Hern{\'a}ndez-Lobato et al.).
    \item \textbf{Reproducibility Stability}: The variation between models trained with a fixed $T$ but different random seeds, capturing instability due to the algorithm's inherent randomness and limited data (Wager et al.; Lopes; Lange et al.).
\end{enumerate}
More precisely, Lange et al.\ model how reproducibility stability converges as $T$ increases.

These two notions are related: convergence stability reflects a component of reproducibility stability, because differences between models with different tree counts arise not only from systematic bias but also from variance at fixed $T$. 
The relative difference in OOB scores used in our plateau criterion naturally captures this stochasticity. 
By requiring the relative change in the OOB score to fall below $\varepsilon$, we stop increasing $T$ once additional trees yield no meaningful improvement. This accounts for both bias reduction and finite-sample variability.
Hence, OOB randomness is a feature rather than a drawback: the plateau condition tracks convergence and accounts for reproducibility-related fluctuations.

From the standpoint of statistical decision-making, the methods employ distinct frameworks.
Latinne et al.\ base their decision on the McNemar test, which compares the numbers of instances misclassified by ensemble but not the other when contrasting two ensembles of different sizes.
Hern\'andez-Lobato et al.\ propose a method for estimating the minimum ensemble size $T$ such that the predictions of a finite ensemble coincide, on average,
with probability at least $\alpha$ (close to 1), with those of a hypothetical infinite ensemble ($T \to \infty$).
Lange et al.\ distinguish three types of stability: prediction stability, variable importance stability, and selection stability (e.g., for selecting the best individuals or the most important features). 

Hern{\'a}ndez-Lobato et al.\ provide an asymptotic approximation:
\begin{equation}
T(\alpha) \approx \left( \frac{f(1/2) \cdot C}{1 - \alpha} \right)^2,
\quad C = \int_{-\infty}^{0} \Phi(z) \, dz
=\frac{1}{\sqrt{2\pi}},
\label{eq:T_alpha}
\end{equation}
where $\Phi(z)$ is the cumulative distribution function of a standard Gaussian distribution.
The key component $ f(\pi_1) $ is the probability density function of the random variable $\pi_1(x)$ --- the probability that a single base classifier predicts the positive class for an input feature vector $x$.
The density $f(\pi_1) $ is estimated empirically from an ensemble of base classifiers trained on the training set, with $\pi_1(x)$ estimated from their outputs on validation/test instances or via OOB predictions.
For a set of held-out or OOB data points, $\pi_1(x) $ for each point is approximated by the fraction of base classifiers predicting the positive class. 
A kernel density estimator or a histogram is then used to obtain $\hat{f}(\pi_1)$,
from which the critical value $\hat{f}(1/2)$ is extracted.

The value $f(1/2)$ quantifies the concentration of ``borderline'' instances (where $ \pi_1(x) \approx 1/2 $), which are the most uncertain and dictate the convergence rate of the ensemble.
This theoretical result provides a principled guideline: the required ensemble size scales with the square of the density of hard examples and inversely with the square of the acceptable uncertainty $1-\alpha$.
However, the analysis is subject to two main assumptions: (1) the problem is binary classification with majority voting, and (2) the normal approximation (via the Central Limit Theorem) is valid for large $T$.
From \eqref{eq:T_alpha} it follows that $1-\alpha = O\!\left(T^{-1/2}\right)$ as $T\to\infty$,
i.e., the uncertainty level in Hern\'andez-Lobato et al.\ decreases at the same rate as the conditional standard deviation of the Random Forest output in \eqref{eq:var}.

Lopes~\cite{Lopes2016Sharp} pursued a related direction for binary classification and studied the \emph{algorithmic} variance (i.e., due only to the randomized training algorithm with fixed training set $D$) of the ensemble misclassification error $ERR_T$.
This quantity can be viewed as a class-conditional population counterpart of $1-\mathrm{accuracy}$.
In particular, under the setting considered there, the author derived a first-order asymptotic expansion for the conditional mean ensemble error of the form
(see~\cite{Lopes2016Sharp}, Lemma~1, formula~(7))
\begin{equation}
\mathbb{E}[\mathrm{ERR}_T\mid D]
=
\mathrm{ERR}_\infty + \frac{f'(1/2)}{8T}+o(T^{-1}),
\quad T\to\infty,
\label{eq:lopes_mean}
\end{equation}
and also obtained a sharp bound for the variance
\begin{equation}
\mathrm{Var}\!\left[ERR_T \mid D\right]
\le
\frac{f(1/2)^2}{4T}
+
o(T^{-1}),
\qquad T\to\infty,
\label{eq:lopes_bound}
\end{equation}
which implies that the corresponding algorithmic standard deviation decays as $O(T^{-1/2})$.

In the same line of work, Lopes considers relative stopping conditions of the form
\begin{equation*}
\sqrt{\mathrm{Var}\!\left[ERR_T \mid D\right]} \le \varepsilon\, ERR_T,
\end{equation*}
which may be viewed as a coefficient-of-variation-type criterion: the scale of the remaining algorithmic fluctuations is required to be small relative to the error level itself. In this sense, our plateau criterion is similar in spirit, since it also monitors whether the remaining score fluctuations are small in relative terms, but does so directly through successive changes of the chosen score.

Lange et al. measure stability as the agreement among 10 independently trained Random Forests (each run with a different random seed). 
Depending on the task, the compared quantities are the predicted values (for prediction stability), feature  importance (for VIM stability), or the top 15 percent of individuals or top 5 percent of features for selection stability. 
The agreement is quantified using the Intraclass Correlation Coefficient (ICC) for regression and variable importance, and Fleiss' $\kappa$ for classification and selection stability. 
Both ICC and Fleiss' $\kappa$ typically lie in $[0,1]$, with $1$ indicating perfect agreement (negative values may occur, indicating agreement worse than chance).
Stability is calculated at five fixed numbers of trees: $T = 250, 500, 750, 1000$, and $2000$, 
which serve as anchor points (knots) for subsequent extrapolation. 
The relationship between the number of trees and stability is then modeled using a logistic-type function
\begin{equation}
    \hat{s} = \frac{1}{1+\left(\frac{\theta_1}{T}\right)^{\theta_2}},
    \label{eq:logistic}
\end{equation}
where the parameters $\theta_1$ and $\theta_2$
are fitted to these points via nonlinear least squares. 
This fit serves as an empirical saturation curve; its asymptotic rate is model-driven and need not coincide with the $O(T^{-1/2})$ scaling.
The equation \eqref{eq:logistic} is used to extrapolate stability for much larger values of $T$. 
The optimal number of trees is defined as the point where increasing $T$ by 10 trees improves the stability by no more than $10^{-6}$.

Hernández-Lobato et al.\ considered traditional, i.e. high-density datasets $n \gg p$, and for the confidence level $\alpha=0.99$, the estimated ensemble sizes vary widely across tasks, ranging from tens to thousands of trees.
Lange et al.\ focused on datasets with $n \ll p$ from bioinformatics, where the number of features reached up to $p = 139{,}101$ with $n = 1{,}063$ observations. 
The estimated optimal number of trees reached up to $T = 708{,}000$ for selection stability in one dataset.
It should be noted that although the stability measures used (ICC and Fleiss’ $\kappa$) are statistically sound and widely accepted, they are not intuitive for users and are difficult to interpret.

Moreover, the method depends on several predefined parameters, such as using 10 repeated forests, thresholds of 15 and 5 percent in the selection stability, and 10 trees added to the current ensemble. 
The criterion that the stability gain per 10 additional trees does not exceed a small tolerance ($10^{-6}$) does not guarantee stability values close to 1; 
for instance, a stability of only $0.845$ was achieved with $T = 137{,}000$ trees in one dataset.
Furthermore, for such non-intuitive stability metrics, it is not clear how to interpret in practice whether increasing $T$ from, say, $0.845$ to $0.9$ or to $0.99$ is really necessary.

Our relative OOB-score criterion shares a similar limitation: satisfying a tolerance $\varepsilon$ does not by itself guarantee closeness of the current score to its limiting value, and convergence to that limit may remain slow. 
However, below we derive a theoretical link between the relative OOB-score difference and the gap to the limiting score. 
Moreover, because the criterion is formulated in relative units, the tolerance $\varepsilon$ remains directly interpretable in practice as an approximate indicator of the number of significant digits of the Random Forest score.

Additionally, the approaches of Latinne et al., Hern{\'a}ndez-Lobato et al., and Lange et al. do not address the joint tuning of the number of trees together with other critical Random Forest hyperparameters, such as \param{mtry} or tree depth (Bernard et al.~\cite{Bernard2009Influence}, Scornet~\cite{Scornet2017Tuning}).
Our work does not aim to replace such ensemble‑stabilization methods, as we do not focus on stabilizing VIMs or feature/sample selection. 
Instead, we address the primary tuning of Random Forest, after which model stability depends primarily on the number of trees, which can be increased further to meet additional VIM stability criteria.
\subsection{Hyperparameter optimization approaches}
Hyperparameter optimization (HPO) encompasses several fundamental approaches that are well-documented in major surveys and books:
Feurer and Hutter~\cite{Feurer2019Hyperparameter}, 
Bischl et al.~\cite{Bischl2023Hyperparameter},
Bartz et al.~\cite{bartz2023hyperparameter}.
The main HPO strategies can be categorized as follows:
\begin{itemize}
    \item \textbf{Exhaustive search:} Grid search and its variants.
    \item \textbf{Stochastic search:} Random search and its derivatives.
    \item \textbf{Sequential model-based optimization (SMBO):} 
        \begin{itemize}
            \item Bayesian optimization with Gaussian processes (e.g., Spearmint, SMAC).
            \item Tree-structured Parzen estimator (TPE) (e.g., Hyperopt, Optuna).
        \end{itemize}
    \item \textbf{Multi-fidelity methods:} Techniques that use low-fidelity approximations to speed up search, such as Hyperband.
\end{itemize}

The most basic distinction lies between exhaustive grid search and stochastic random search (Bergstra and Bengio~\cite{Bergstra2012Random}). 
Grid search performs a complete, non-stochastic enumeration of all hyperparameter combinations within a predefined discrete grid, requiring the user to specify exact ranges and discretization steps for each parameter. 
A fundamental limitation of this approach is its inability to adapt search density based on parameter importance --- a property known as \emph{tunability} 
(Probst et al.~\cite{Probst2019Tunability}).
Consequently, grid search often allocates excessive computational resources to less influential parameters or regions of the search space and may miss optimal values of critical parameters that lie between grid points.
This inefficiency is compounded by the practical reality that users rarely possess sufficient prior knowledge to design optimal grids.

Random search addresses this by sampling configurations independently from specified distributions, thereby exploring the parameter space more flexibly. 
Each parameter can be sampled across its entire continuous range, increasing the probability of finding near-optimal regions for important parameters. 
The approach also enables trivial parallelization. 
However, it remains an undirected strategy, as each trial is independent and does not leverage information from previous evaluations. 
It also ignores potential relationships between parameters --- for instance, between tree depth and minimum samples per leaf --- meaning that many sampled combinations may be inherently suboptimal.

This limitation is overcome by Sequential Model-Based Optimization methods, which perform iterative, informed search. 
These algorithms generate new hyperparameter configurations based on the performance history of previous trials, creating a directed process that converges toward promising regions of the search space. 
SMBO methods, such as the Tree-structured Parzen Estimator (Bergstra et al.~\cite{Bergstra2011Algorithms}), adaptively construct probabilistic surrogate models to guide sampling, effectively avoiding exhaustive exploration of low-performance areas.

Gaussian process-based SMBO, as implemented in 
Spearmint (Snoek et al.~\cite{Snoek2012Practical}) 
and SMAC (Hutter et al.~\cite{Hutter2011Sequential}), 
represents another prominent approach. 
Although GP-based methods provide better uncertainty quantification and theoretical guarantees in low-dimensional, continuous parameter spaces (typically up to 10--20 hyperparameters), they scale poorly to higher-dimensional spaces in terms of computational complexity and have difficulty handling discrete and categorical parameters.
For tuning Random Forest --- which involves mixed parameter 
types including discrete (\param{max\_depth}) and continuous 
variables --- TPE has proven particularly effective 
in practice due to its robustness and scalability.

An important advancement in HPO is trial pruning---the early termination of unpromising configurations based on intermediate performance metrics.
Techniques such as Hyperband (Li et al.~\cite{Li2018Hyperband})
and Bayesian Optimization and Hyperband (BOHB, Falkner et al.~\cite{Falkner2018BOHB})
implement aggressive pruning schedules.
This can accelerate optimization compared to plain TPE, but may also discard configurations that would perform well only at larger budgets.
Moreover, these algorithms do not resolve the uncertainty in choosing $T_{\max}$, and their main purpose is computational acceleration rather than improving the accuracy and stability of the resulting hyperparameter estimates.

In contrast, we use pruning in a different sense: not primarily as a mechanism for computational acceleration, but to exclude from the final comparison those trials for which a plateau has not been confirmed.
Thus, in our method pruning acts mainly as a quality filter, ensuring that the best-trial selection is performed only among configurations that satisfy the plateau criterion.
We use a BOHB-like baseline among the baselines to compare computational cost and the extent of pruning.

Among contemporary HPO frameworks, two widely used TPE-based tools are
Hyperopt (Bergstra et al.~\cite{Bergstra2015Hyperopt}) and
Optuna (Akiba et al.~\cite{Akiba2019Optuna}).
We selected Optuna for its native pruning support and straightforward API, which facilitated the implementation of our custom pruning logic based on plateau detection.
Although both frameworks can support our algorithm, Optuna’s design made the integration more convenient in practice.

\section{Triplet-based plateau search algorithm}
\label{sec:algorithm}
\subsection{Proposed Method}
\begin{figure*}[t]
\centering
\begin{tikzpicture}[
    point/.style={circle, inner sep=2pt},
    gap/.style={decorate, thick},
    dimension/.style={{Stealth[length=5pt]}-{Stealth[length=5pt]}, thin, black},
    dimension_in/.style={{}, thin, black},
    dimension_out/.style={-{Stealth[length=5pt]}, thin, black},
    extension/.style={shorten <=-2pt, shorten >=-2pt, thin, black}
]
\begin{axis}[
    width=\textwidth,
    height=0.4\textwidth,
    xlabel={$T$},
    ylabel={Score},
    axis lines=left,
    axis line style={-Stealth, thick},
    enlargelimits=true,
    xtick=\empty,
    ytick=\empty,
    xmin=0, xmax=100,
    ymin=0.5, ymax=1.01,
    domain=0:100,
    samples=200
]
\addplot[thick, blue] {0.5 + 0.5*(1 - exp(-0.08*x))};
\draw[dashed, very thick, gray] (axis cs:0,1) -- (axis cs:100,1);
\node[above, font=\small] at (axis cs:5,1.01) {$S_{\infty}$};
\def\xLl{3}
\def\xBl{6}
\def\xRl{11}
\pgfmathsetmacro{\yLl}{0.5 + 0.5*(1 - exp(-0.08*\xLl))}
\pgfmathsetmacro{\yBl}{0.5 + 0.5*(1 - exp(-0.08*\xBl))}
\pgfmathsetmacro{\yRl}{0.5 + 0.5*(1 - exp(-0.08*\xRl))}
\draw[extension] (axis cs:\xLl,\yLl) -- (axis cs:\xLl-1.5,\yLl);
\draw[extension] (axis cs:\xBl,\yBl) -- (axis cs:\xLl-1.5,\yBl);
\draw[dimension] (axis cs:\xLl-1.5,\yLl) -- (axis cs:\xLl-1.5,\yBl)
    node[midway, left=1] {$\text{plat}_L > \varepsilon$};
\draw[extension] (axis cs:\xRl,\yRl) -- (axis cs:\xRl+1,\yRl);
\draw[extension] (axis cs:\xBl,\yBl) -- (axis cs:\xRl+1,\yBl);
\draw[dimension] (axis cs:\xRl+1,\yRl) -- (axis cs:\xRl+1,\yBl)
    node[midway, right=1] {$\text{plat}_R > \varepsilon$};
%
\node[point, label=below right:L, fill=red] (L1) at (axis cs:\xLl,\yLl) {};
\node[point, label=below right:B, fill=red] (B1) at (axis cs:\xBl,\yBl) {};
\node[point, label=right:\>R, fill=red] (R1) at (axis cs:\xRl,\yRl) {};
\node[align=center] at (axis cs:15, 0.6) {shift right};
%
\def\xL{30}
\def\xB{40}
\def\xR{53}
\pgfmathsetmacro{\xBR}{0.5*\xB + 0.5*\xR}
\pgfmathsetmacro{\yL}{0.5 + 0.5*(1 - exp(-0.08*\xL))}
\pgfmathsetmacro{\yB}{0.5 + 0.5*(1 - exp(-0.08*\xB))}
\pgfmathsetmacro{\yR}{0.5 + 0.5*(1 - exp(-0.08*\xR))}
\draw[extension] (axis cs:\xL,\yL) -- (axis cs:\xL-2,\yL);
\draw[extension] (axis cs:\xB,\yB) -- (axis cs:\xL-2,\yB);
\draw[dimension_in] (axis cs:\xL-1.5,\yL) -- (axis cs:\xL-1.5,\yB)
    node[midway, left=1.5] {$\text{plat}_L > \varepsilon$};
\draw[dimension_out] (axis cs:\xL-1.5,\yL-0.08) -- (axis cs:\xL-1.5,\yL);
\draw[dimension_out] (axis cs:\xL-1.5,\yB+0.07) -- (axis cs:\xL-1.5,\yB);
\draw[extension] (axis cs:\xB,\yB) -- (axis cs:\xBR+1,\yB);
\draw[extension] (axis cs:\xR,\yR) -- (axis cs:\xBR-1,\yR);
\draw[dimension_in] (axis cs:\xBR,\yB) -- (axis cs:\xBR,\yR);
\draw[dimension_out] (axis cs:\xBR, \yB-0.07) -- (axis cs:\xBR,\yB);
\draw[dimension_out] (axis cs:\xBR, \yR+0.1) -- (axis cs:\xBR,\yR)
    node[midway, right=1.5] {$\text{plat}_R \leq \varepsilon$};
%
\node[point, label=below:L, fill=red] (L2) at (axis cs:\xL,\yL) {};
\node[point, label=below:B, fill=green] (B2) at (axis cs:\xB,\yB) {};
\node[point, label=below:R, fill=green] (R2) at (axis cs:\xR,\yR) {};
\node[align=center] at (axis cs:41,0.86) {stay};
%
\def\xLr{70}
\def\xBr{81}
\def\xRr{95}
\pgfmathsetmacro{\xLBr}{0.5*\xLr + 0.5*\xBr}
\pgfmathsetmacro{\xBRr}{0.5*\xBr + 0.5*\xRr}
\pgfmathsetmacro{\yLr}{0.5 + 0.5*(1 - exp(-0.08*\xLr))}
\pgfmathsetmacro{\yBr}{0.5 + 0.5*(1 - exp(-0.08*\xBr))}
\pgfmathsetmacro{\yRr}{0.5 + 0.5*(1 - exp(-0.08*\xRr))}
\node[point, label=below:L, fill=green] (L3) at (axis cs:\xLr,\yLr) {};
\node[point, label=below:B, fill=green] (B3) at (axis cs:\xBr,\yBr) {};
\node[point, label=below:R, fill=green] (R3) at (axis cs:\xRr,\yRr) {};
\node[point, label=above:$\text{plat}_L \leq \varepsilon$] () at (axis cs:\xLBr,\yBr) {};
\node[point, label=above:$\text{plat}_R \leq \varepsilon$] () at (axis cs:\xBRr,\yBr) {};
\node[align=center] at (axis cs:81,0.88) {shift left};
\end{axis}
\end{tikzpicture}
\caption{Three scenarios in plateau search for tree count in Random Forest hyperparameter optimization (limiting score $S_{\infty}$ --- plateau level): 
shift right --- early stage, significant improvements, 
stay --- transition to plateau, diminishing returns, and 
shift left --- plateau region, minimal improvements. 
Points $L$, $B$, $R$ represent consecutive evaluations with vertical gaps $\text{plat}_L$ and $\text{plat}_R$.}
\label{fig:plateau}
\end{figure*}
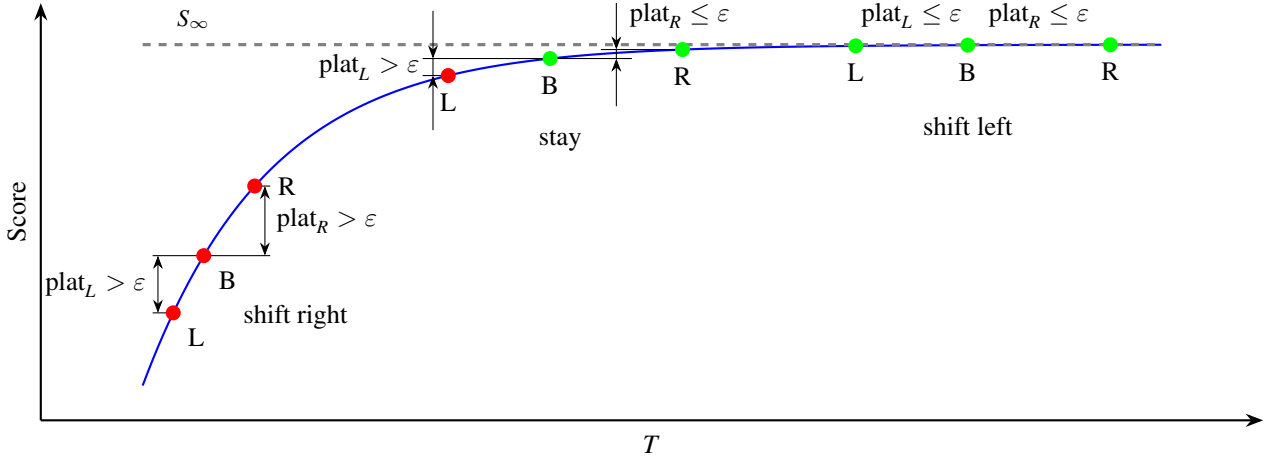
\begin{algorithm}[t]
\SetAlgoLined
\SetKwComment{tcp}{\%}{}
\KwIn{Dataset $(X, y)$, scale factor $\textnormal{sf} > 1$, tolerance $\varepsilon > 0$,
maximum trials $n_\text{trials}$, initial tree count $T_0$,
other hyperparameter search spaces;}
\KwOut{Optimal hyperparameters $\text{Params}_\text{best}$, near-optimal tree count $T = B_\text{best}$;}
\renewcommand{\thealgocf}{}
\caption{Triplet-Based OOB Plateau Search for Random Forest Hyperparameter Tuning}
\label{alg:tune_rf_oob_plateau}
\textbf{Initialization:} $[L, B, R] \gets [T_0/\textnormal{sf}, T_0, T_0 \cdot \textnormal{sf}]$\;
\For(\tcp*[f]{optimization phase}){$k \gets 1$ \textbf{\upshape to} $n_\textnormal{trials}$}{ 
    Sample $\text{Params}_k$ via Optuna's Bayesian TPE\;
    Fit RF on $(X, y)$ with $\text{Params}_k$ at tree counts $L, B, R$ and obtain OOB scores $S_L, S_B, S_R$\;
    $\text{plat}_L \gets |S_L-S_B|/|S_B|$\;
    $\text{plat}_R \gets |S_R-S_B|/|S_B|$\;
    \If(){$\textnormal{plat}_R \leq \varepsilon$ \textbf{\upshape and} $\texttt{\upshape SHOULD\_PRUNE}(S_R)$}{ 
        $\texttt{PRUNE}(k)$; \label{prune_classic} \tcp*[f]{prune trial} \\
        \textbf{continue}\;
    }
    \If(\tcp*[f]{plateau reached}){$\textnormal{plat}_R \leq \varepsilon$}{
        $\text{Score}_k \gets S_B$; \tcp*[f]{register trial} \\
        $[L_k, B_k, R_k] \gets [L, B, R]$; \tcp*[f]{save triplet} \\
        \If(\tcp*[f]{both sides plateau}){$\textnormal{plat}_L \leq \varepsilon$}{
            $[L, B, R] \gets [L/\textnormal{sf}, L, B]$; \tcp*[f]{shift left} \\
        }
        $[L_k^\text{next}, B_k^\text{next}, R_k^\text{next}] \gets [L, B, R]$; \tcp*[h]{save next}\\
    }
    \Else{
        $[L, B, R] \gets [B, R, R \cdot \textnormal{sf}]$; \tcp*[f]{shift right}\\
        $\texttt{PRUNE}(k)$ \label{prune_plateau} \;
        \textbf{continue}\;
    }
}
\If(\tcp*[f]{$\nexists$ good trial}){$\textnormal{Score}_k$ {\upshape is undefined} $\forall k$}{
    \Return $\text{Params}_{\text{best}}\gets NULL$, $B_\text{best} \gets NULL$\;
}
$k_\text{best} \gets \arg\mathrm{best}\ \text{Score}_k$; \tcp*[f]{get best trial}\\
$\text{Params}_{\text{best}} \gets \text{Params}_{k_\text{best}}$, $B_\text{best} \gets B_{k_\text{best}}$\;
\hrulefill\\
\textbf{Initialization:} $[L, B, R] \gets [L_{k_\text{best}}^\text{next}, B_{k_\text{best}}^\text{next}, R_{k_\text{best}}^\text{next}]$\;
\For(\tcp*[f]{revisit phase}){$k \gets 1$ \textbf{\upshape to} $n_\textnormal{trials}$}{
    Fit RF on $(X, y)$ with $\text{Params}_\text{best}$ at tree counts $L, B, R$ and obtain OOB scores $S_L, S_B, S_R$\;  
    $\text{plat}_L \gets |S_L-S_B|/|S_B|$\;
    $\text{plat}_R \gets |S_R-S_B|/|S_B|$\;
    \If(\tcp*[f]{still right plateau}){$\textnormal{plat}_R \leq \varepsilon $}{ 
        \If(\tcp*[f]{can reduce further}){$\textnormal{plat}_L \leq \varepsilon $}{
            $B_\text{best} \gets B$\;
            $[L, B, R] \gets [L/\textnormal{sf}, L, B]$; \tcp*[f]{shift left}\\
        }
        \Else(\tcp*[f]{found minimal tree count}){
            \Return $\text{Params}_\text{best}$, $B$\;
        }
    }
    \Else(\tcp*[f]{next plateau lost}){ 
        \Return $\text{Params}_\text{best}$, $B_\text{best}$\;
    }
}
\Return{$\textnormal{Params}_\textnormal{best}$, $B_\textnormal{best}$\;}
\end{algorithm}
Conventional HPO tunes all hyperparameters jointly. 
To effectively incorporate the number of trees --- which is intrinsically linked to the stochastic sampling process of Random Forest---an iterative, model-based approach is especially suitable. 
As other hyperparameters converge toward optimal values through successive trials, the adjusted tree count from previous iterations increasingly represents improved configurations, causing it to stabilize near its optimum. 
Our method implements this within Optuna TPE sampler, maintaining a fixed search space for standard parameters and dynamically updating the number of trees based on a plateau detection criterion.

The core idea of our algorithm is to perform, within each TPE trial, three Random Forest trainings with different numbers of trees. 
We define a triplet of tree counts related by a fixed multiplicative scaling factor $\textnormal{sf} > 1$: 
$L = B / \textnormal{sf}$, $B$, and $R = B \cdot \textnormal{sf}$. 
Random Forest training can be performed in \param{warm\_start}=``True'' mode, meaning each subsequent ensemble does not start from scratch but incrementally adds the required number of trees to the previous model.
After training each ensemble, we compute its OOB score on the full dataset: $S_L$, $S_B$, and $S_R$.

We then evaluate plateau conditions for consecutive points by computing the relative differences:
\begin{equation}
    \text{plat}_L = \left|\frac{S_B-S_L}{S_B}\right|, \qquad
    \text{plat}_R = \left|\frac{S_R-S_B}{S_B}\right|.
    \label{plat_L_R}
\end{equation}
Based on these conditions, we distinguish three primary scenarios for the current tree count $B$ (Fig.~\ref{fig:plateau}):
\begin{enumerate}
    \item $B$ is insufficient: $\text{plat}_L > \varepsilon$, $\text{plat}_R > \varepsilon$,
    \item $B$ is sufficient: $\text{plat}_L > \varepsilon$, $\text{plat}_R \leq \varepsilon$,
    \item $B$ is excessive: $\text{plat}_L \leq \varepsilon$, $\text{plat}_R \leq \varepsilon$,
\end{enumerate}
where $\varepsilon$ is a predefined tolerance. 
These conditions indicate whether the OOB score has stabilized (reached a plateau) to the left or right of $B$.

The algorithm adjusts the triplet according to the scenario:
\begin{itemize}
    \item If $B$ is insufficient (case 1), shift the entire triplet to the right by multiplying by $\textnormal{sf}$: 
        \begin{equation*}
            [L, B, R] \gets [L \cdot \textnormal{sf}, B \cdot \textnormal{sf}, R \cdot \textnormal{sf}],
        \end{equation*}
    \item If $B$ is sufficient (case 2), keep the current $B$,        
    \item If $B$ is excessive (case 3), shift left by dividing by $\textnormal{sf}$: 
        \begin{equation*}
            [L, B, R] \gets [L / \textnormal{sf}, B / \textnormal{sf}, R / \textnormal{sf}].
        \end{equation*}
\end{itemize}
Thus, $B$ represents the approximate plateau onset point.

A fourth, less typical case can occur due to random OOB-score fluctuations, especially for small $B$: $\text{plat}_L \leq \varepsilon$, $\text{plat}_R > \varepsilon$. 
In principle, one could treat this case symmetrically with case 2, i.e., as a stay-case. 
In the present work, however, we group it together with case 1 and interpret it as a shift-right condition, since the score has not yet stabilized to the right of $B$, which suggests that $B$ may still be insufficient.
We prune the current trial in these cases, meaning that not even the OOB score of the largest forest ($R$ trees) is stable enough. 
Grouping case 4 with case 1 partly counteracts the usual tendency of early-stopping rules to underestimate the required tree count once $\text{plat}_L \le \varepsilon$ is already satisfied. 
This can make the estimated number of trees somewhat larger and the pruning induced by the absence of a confirmed plateau more aggressive. 
On the other hand, such a moderate upward bias may also be viewed as favoring a more stable forest.

To avoid accumulation of rounding errors, we update the triplet by reassigning two values and computing the third via multiplication or division by $\textnormal{sf}$:
\begin{itemize}
    \item For a right shift: $[L, B, R] \gets [B, R, R \cdot \textnormal{sf}]$,
    \item For a left shift: $[L, B, R] \gets [L / \textnormal{sf}, L, B]$.
\end{itemize}
After each update, the values are rounded to the nearest integer since the number of trees must be an integer.
In cases 2 and 3, the trial returns the OOB score $S_B$. 
Although $S_R$ might be slightly higher than $S_B$, the algorithm prioritizes finding the smallest adequate ensemble, and since the relative difference is within tolerance, $B$ is chosen as the final size.
Besides plateau-based pruning, the method also supports the standard Optuna pruning mechanism, invoked through \texttt{trial.should\_prune()}.

After completing all trials, the trial with the best OOB score $S_B$ is selected. 
Even though the process should stabilize, a user may specify too few trials, leaving the optimization far from convergence. 
Hence, the best trial might correspond to an excessive tree count. 
To correct this, we apply a simple postprocessing step: we successively shift the best trial’s triplet to the left as long as condition~3 holds and keep the other hyperparameters fixed (revisit phase).
For this purpose, we store both the OOB scores $\text{Score}_k$ and the corresponding triplets 
(original $[L_k, B_k, R_k]$ and shifted $[L_k^\text{next}, B_k^\text{next}, R_k^\text{next}]$)
for each trial during the optimization, enabling us to identify the optimal ensemble size even when the best trial initially has an excessive number of trees.

For clarity, the detailed pseudocode follows the workflow below:
\begin{enumerate}
    \item Initialize the triplet $(L,B,R)$.
    \item Sample the remaining Random Forest hyperparameters and evaluate the OOB score at $L$, $B$, and $R$.
    \item Depending on the plateau conditions, shift right to $R$, stay at $B$, or shift left to $L$.
    \item Right shifts lead to pruning of the trial; stay and left-shift cases are registered together with the corresponding triplet.
    \item After optimization, revisit the best trial and continue shifting left as long as the plateau remains satisfied.
\end{enumerate}

\subsection{Discussion of Key Features}
\begin{figure*}[t]
\centering
\includegraphics[width=\linewidth]{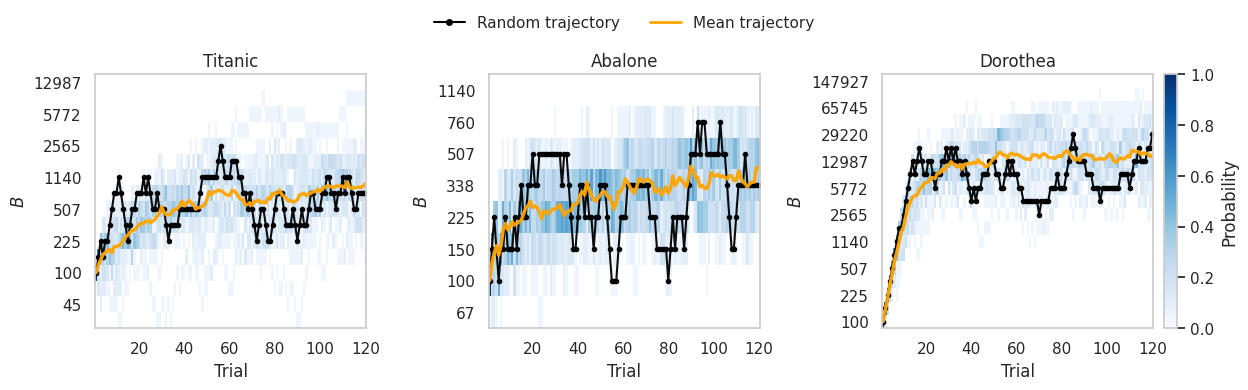}
\caption{Examples of random and mean trajectories of the central triplet point $B$ across HPO trials for three datasets, with the mean computed over 20 runs. 
The background colormap shows the empirical frequency of the corresponding tree count at each trial. 
Here $T_0=100$ and $\text{sf}=1.5$; in the left panel, $\varepsilon=3\times 10^{-3}$, and in the other two panels, $\varepsilon=10^{-3}$.
}
\label{fig:trajectories}  
\end{figure*}
We employ a multiplicative update of the tree count, analogous to the geometric progression used by Oshiro (who effectively used $\textnormal{sf}=2$), rather than an additive increment as in Latinne et al.~\cite{Latinne2001Limiting} and Lange et al.~\cite{Lange2025optRF}.
Accordingly, our grid of candidate tree counts is fixed once the initial value $T_0$ and the scale factor $\textnormal{sf}$ are chosen. 
This design reduces the risk of premature stopping (i.e., selecting too few trees) before the plateau is reached. 
At the same time, if the grid grew faster than $O(\textnormal{sf}^{\,j})$ as $j\to\infty$, it would lead to triplets with excessively large $R$, making the procedure computationally prohibitive.

Each trial evaluates three forest sizes ($L$, $B$, and $R$) instead of one; 
nevertheless, with \param{warm\_start} the cumulative training cost is approximately that of growing the forest up to $R=B\cdot \textnormal{sf}$, 
i.e. $O(R)$ rather than $O(L+B+R)$.
Thus, relative to a standard TPE trial at a comparable tree budget, the additional overhead of the plateau method is mainly due to extra OOB score evaluations and bookkeeping, rather than training three independent ensembles.
In our experiments, this adaptive update scheme typically reaches the stringent tolerance level $\varepsilon=10^{-3}$ substantially faster than the classical TPE's fixed-range setting, which requires specifying an explicit upper bound (e.g., $T_{\max}\approx 2000$) for \param{n\_estimators}.

We note that line~\ref{prune_classic} implements standard Optuna-style pruning, analogous to what is used in the classical TPE baseline: after evaluating the largest budget level (here, $R$), the intermediate result is reported and the trial may be pruned by the chosen pruner.
In contrast, line~\ref{prune_plateau} performs plateau-specific pruning for trials in which a plateau is not reached within the allowed tree budget, and therefore the observed scores are deemed inconclusive for assessing the sampled hyperparameter configuration.

Such trials are discarded from the optimization process: they do not contribute an objective value, are not considered when selecting the best trial, and do not influence subsequent trial generation.
The implementation-level handling of this plateau-specific pruning rule is described in Section~\ref{subsec:implementation_details}.
In particular, a no-plateau trial is terminated without returning an objective value and without reporting the baseline OOB score $S_B$ to Optuna.
We further examine whether such ``inconclusive'' trials consume a substantial portion of the overall budget $n_{\text{trials}}$, and we compare this mechanism to pruning in the Hyperband baseline under typical settings.

The relative-error condition~\eqref{plat_L_R} controls both how far the algorithm advances into the plateau region and how much the reproducibility stability of the ensemble improves. 
In this sense, it acts simultaneously on bias and variance.
Below, we show that the variance of the corresponding relative difference decays at rate $O(T^{-1})$, first for the signed version of~\eqref{plat_L_R} and then after passing to its absolute-value form. 
The signed version is used only as an auxiliary device in the theoretical analysis. 
In the algorithm itself, however, the absolute value is indispensable: without it, random OOB-score fluctuations would systematically favor right shifts of the triplet.

Indeed, when $\varepsilon \approx 0$ and $S_L \approx S_B \approx S_R$, the signed differences
$(S_B-S_L)/S_B$ and $(S_R-S_B)/S_B$ are approximately symmetric around zero, so the four possible sign combinations occur with roughly equal probability.
Under our update rule, two of these combinations trigger a shift to the right (cases 1 and 4), only one triggers a shift to the left (case 3), and one results in staying at $B$ (case 2).
Using absolute values removes this artificial drift by making the decision depend on the magnitude of the change rather than its sign.

Even then, the plateau-onset estimate $B$ may oscillate across trials rather than settling at a single value (see Fig.~\ref{fig:trajectories}).
Such variability arises from exploring different hyperparameter configurations (via TPE) and from the inherent stochasticity of Random Forest training as reflected in the OOB score.
Importantly, our method does \emph{not} rely on a dedicated statistical stopping test that would certify that the current tree count $T=B$ is sufficient at a prescribed confidence level.
Implementing such a test would typically require additional replicated training at fixed hyperparameters (e.g., multiple forests with different random seeds and/or bootstrap resampling of OOB predictions) in order to estimate variability and construct a confidence interval.
Besides increasing the per-trial computational cost, the resulting decision would depend on extra user-chosen parameters (e.g., the confidence level and the number of replications).

To avoid this additional overhead within each trial, we do not introduce an ``artificial'' inner resampling loop (such as bootstrap-based variance estimation) on top of the already stochastic OOB evaluation.
Instead, we exploit the \emph{natural} variability provided by the multi-trial HPO protocol itself: across trials, the procedure observes a range of configurations and OOB realizations, and the best-performing trial is selected in the standard Optuna design.
As a result, choosing a sufficiently large trial budget $n_{\text{trials}}$ becomes particularly important for stable performance.
This is further motivated by the fact that some trials may be terminated after repeated right shifts when no plateau is reached within the allowed budget; such trials are treated as inconclusive and do not contribute to the pool of completed trials used for selecting the final configuration.

Another important practical feature of the proposed method is that, once the HPO process has entered a sufficiently settled regime, the baseline value in the next trial, $B_{k+1}$, is likely to remain close to $B_k$, or at least to one of its neighboring values $L_k$ or $R_k$. Therefore, at each trial $k$ the algorithm evaluates the OOB score only for the nested forests corresponding to $L_k$, $B_k$, and $R_k$, rather than traversing all intermediate ensemble sizes such as $T_0,T_0\cdot \textnormal{sf}, \ldots,L_k/\textnormal{sf}^2,L_k/\textnormal{sf}$, as in monotone early-stopping schemes. Although the dominant computational cost is usually attributed to tree construction itself, these intermediate score evaluations may also be non-negligible in practice. Moreover, when a right shift is required, we do not continue growing the current forest under the same hyperparameter configuration. Instead, the enlarged working tree scale is explored already in the next trial, together with a newly sampled set of the remaining hyperparameters. This avoids overinvesting computation into a potentially unfavorable configuration and also explains the motivation for pruning trials in which no plateau has been confirmed.

Thus, the reliability of our plateau-based algorithm is achieved by aggregating evidence across the optimization process itself: a sufficiently rich sequence of trials provides many noisy comparisons which, in effect, ``vote'' for or against increasing $T$.
Thus, the procedure depends on using a sufficiently large number of trials, but avoids introducing extra tuning knobs such as confidence levels or a fixed number of independent refits (e.g., 10) as in stability-based approaches.
This is an intentional trade-off: it avoids the computational expense of dedicated repeated runs by leveraging the multiplicity of Random Forest trainings performed across the sequential trials of the TPE optimization process itself, and can be viewed as seeking an empirical equilibrium region in which left and right shifts balance out.

This inherent stochasticity also affects the revisit phase, albeit without the additional variation from hyperparameter sampling.
In particular, the procedure that repeatedly shifts the triplet to the left until case~3 is first violated can be viewed as a random walk with an absorbing stopping condition.
As a result, its outcome can be highly variable and may deviate from the equilibrium region of the optimal tree count.
Rather than providing a robust estimate, this phase functions more as an aggressive, one-time reduction of the ensemble size.
Therefore, it may be advisable to disable this phase entirely, particularly when the number of optimization trials is already large or when the tuned model is intended to subsequently grow the forest further to stabilize variable-importance estimates.
\subsection{Theoretical interpretation of the plateau criterion}

Let $S_T$ denote the OOB-based score for a Random Forest with $T$ trees, 
and let $S_\infty = \lim_{T\to\infty} S_T$ denote the score of an infinite forest.
This viewpoint is consistent with the classical convergence result of Breiman~\cite{Breiman2001Random}, 
who showed that the population generalization error converges as the number of trees grows; OOB estimates provide an internal proxy for this convergence.

To connect the plateau criterion with the limiting score, we model the residual convergence of the score by the asymptotic expansion
\begin{equation}
S_T = S_\infty + cT^{-\gamma} + o(T^{-\gamma}),
\qquad T\to\infty,
\label{eq:score_tail_alpha}
\end{equation}
for some constants $c\neq 0$ and $\gamma>0$.

\begin{proposition}
\label{prop:plateau_limit_alpha}
Under \eqref{eq:score_tail_alpha}, for $R = \textnormal{sf} \cdot B$,
\begin{equation}
|S_\infty-S_B| \sim \frac{|S_R-S_B|}{1-\textnormal{sf}^{-\gamma}},
\qquad B\to\infty.
\label{eq:ratio_limit_plateau}
\end{equation}
If, in addition,
\begin{equation}
\left|\frac{S_R-S_B}{S_B}\right|\le \varepsilon,
\label{eq:plateau_condition_BR}
\end{equation}
then
\begin{equation}
\left|\frac{S_\infty-S_B}{S_B}\right|
\le
\frac{\varepsilon}{1-\textnormal{sf}^{-\gamma}}+o(1),
\qquad B\to\infty.
\label{eq:plateau_limit_alpha_rel}
\end{equation}
\end{proposition}

Proposition~\ref{prop:plateau_limit_alpha} shows that, under a power-law tail \eqref{eq:score_tail_alpha}, the plateau criterion controls not only the local score change between $B$ and $R=\textnormal{sf}\cdot B$, but also the relative gap between the current score and its limiting value. 
The factor $(1-\textnormal{sf}^{-\gamma})^{-1}>1$ makes the dependence on the scale factor explicit: for fixed $\gamma$, larger values of $\textnormal{sf}$ strengthen the link between the observed plateau gap and the remaining distance to the limit.
This relation also gives a practical interpretation of the tolerance parameter in terms of the limiting score. 
Asymptotically, using a working tolerance $\varepsilon$ in the algorithm corresponds to a relative gap to the limiting score of approximately
\begin{equation}
\left|\frac{S_\infty-S_B}{S_B}\right|
\lesssim
\varepsilon_{\infty} = \frac{\varepsilon}{1-\textnormal{sf}^{-\gamma}}.
\label{eq:eps_inf}
\end{equation}
Equivalently, if one wishes to prescribe the tolerance directly with respect to the limiting score, namely $\varepsilon_{\infty}$, then the algorithmic tolerance should be chosen as
$\varepsilon = \varepsilon_{\infty}(1-\textnormal{sf}^{-\gamma})$.

The asymptotic relation \eqref{eq:eps_inf} becomes more informative once the sign of the tail coefficient $c$ is taken into account. In the natural case of greater-is-better scores, such as accuracy or ROC-AUC, one typically expects asymptotic improvement with $T$, so that $c<0$ and hence $S_B<S_\infty$ for sufficiently large $B$. Then \eqref{eq:eps_inf} yields the informative upper bound
\begin{equation*}
S_B \le S_\infty \lesssim S_B(1+\varepsilon_\infty).
\end{equation*}
Conversely, for error-like scores, such as MSE, one typically expects asymptotic decrease with $T$, so that $c>0$ and therefore $S_\infty \le S_B$ for sufficiently large $B$. In this case, \eqref{eq:eps_inf} gives the informative lower bound
\begin{equation*}
S_B(1-\varepsilon_\infty)\lesssim S_\infty \le S_B.
\end{equation*}

The exponent $\gamma$ determines how tightly the observed plateau gap is linked to the residual distance to the limit. 
A notable special case is $\gamma=1$, which is consistent with Lopes's asymptotic expansion \eqref{eq:lopes_mean} for the conditional mean misclassification error. 
On the other hand, for the observed OOB score, finite-ensemble fluctuations may dominate the residual behavior, in which case an effective exponent closer to $\gamma=1/2$ is also plausible. 
Thus, keeping $\gamma$ explicit allows the proposition to cover both bias-dominated and fluctuation-dominated regimes.
For $\gamma=1$, the correction factor $1-\textnormal{sf}^{-\gamma}$ equals $1/3$ for $\textnormal{sf}=1.5$ and $1/2$ for $\textnormal{sf}=2$; for $\gamma=1/2$, it equals approximately $0.1835$ and $0.2929$.

We now turn to the stochastic behavior of the relative plateau statistic $\text{plat}_R$ in \eqref{plat_L_R} as $B\to\infty$. 
In particular, we study the conditional variance of the corresponding relative score difference given the training set $D$. 
To this end, we use a first-order delta-method approximation (van der Vaart~\cite{Vaart1998Asymptotic}). We begin with a technical lemma.
\begin{lemma}[Delta-method approximations]
\label{lemma:delta}
Let $S_B$ and $S_R$ be the OOB-based scores of two nested Random Forests with $B$ and $R$ trees, and assume that
\begin{equation*}
\left.
\begin{pmatrix}
S_B\\
S_R
\end{pmatrix}
\;\right|\; D
\approx
\mathcal{N}
\left(
\begin{pmatrix}
\mu_B\\
\mu_R
\end{pmatrix},
\begin{pmatrix}
\sigma_B^2 & \sigma_{BR}\\
\sigma_{BR} & \sigma_R^2
\end{pmatrix}
\right)
\end{equation*}
and $\mu_B\neq 0$.
Then the delta method gives
\begin{align}
\mathbb{E}\left[\frac{S_R-S_B}{S_B}\;\middle|\;D\right]
& \approx
\frac{\mu_R-\mu_B}{\mu_B}
+\frac{\mu_R}{\mu_B^3}\sigma_B^2
-\frac{1}{\mu_B^2}\sigma_{BR}, \label{eq:delta_mean_G} \\
\mathrm{Var}\left[ \frac{S_R-S_B}{S_B} \;\middle|\; D \right]
& \approx
\frac{\mu_R^2}{\mu_B^4}\sigma_B^2 +
\frac{1}{\mu_B^2}\sigma_R^2 -
2\frac{\mu_R}{\mu_B^3}\sigma_{BR}.
\label{eq:delta_var_G}
\end{align}
\end{lemma}
More precisely, \eqref{eq:delta_mean_G} combines the zeroth- and second-order contributions in the Taylor expansion of $(S_R-S_B)/S_B$, whereas \eqref{eq:delta_var_G} is given by the first-order delta-method approximation.

Motivated by Lopes's bound~\eqref{eq:lopes_bound}, it is natural to model the conditional variance--covariance terms at the first order as
\begin{equation}
\sigma_B^2 \sim \frac{v}{B}, \quad
\sigma_R^2 \sim \frac{v}{R}, \quad 
\sigma_{BR} \sim \frac{v}{R}, \quad 
B\to\infty,\  R=\textnormal{sf}\cdot B,
\label{eq:remark}
\end{equation}
for some constant $v=v(D)>0$.
Under \eqref{eq:remark}, the corresponding inter-forest correlation satisfies $\sigma_{BR}/(\sigma_B \cdot \sigma_R) \sim 1/\sqrt{\textnormal{sf}}$.
In the special case $\textnormal{sf}=1$, this reduces naturally to $\sigma_B^2=\sigma_R^2=\sigma_{BR}$.

\begin{proposition}
\label{prop:delta_asymptotic}
Assume that $\mu_B,\mu_R \to S_\infty \neq 0$ as $B\to\infty$, with $R=\textnormal{sf}\cdot B$, and that \eqref{eq:remark} holds. 
Then
\begin{equation}
\mathrm{Var}\left[\frac{S_R-S_B}{S_B}\middle| D\right]
\sim
\frac{v}{S_\infty^2}\cdot\frac{1-\textnormal{sf}^{-1}}{B},
\qquad B\to\infty.
\label{eq:prop2_signed}
\end{equation}
\end{proposition}
This follows immediately by substituting $\mu_B,\mu_R \to S_\infty$ and \eqref{eq:remark} into \eqref{eq:delta_var_G} from Lemma~\ref{lemma:delta}.
Under the assumed $O(B^{-1})$ variance--covariance scaling in \eqref{eq:remark}, the conditional variance asymptotics is fully governed by \eqref{eq:delta_var_G}, because the Taylor terms beyond the first order contribute only $o(B^{-1})$.

Thus, \eqref{eq:prop2_signed} gives the asymptotic variance of the signed relative difference $(S_R-S_B)/S_B$. 
As expected, it is of order $O(B^{-1})$, so the corresponding standard deviation is of order $O(B^{-1/2})$. 
The factor $1-\textnormal{sf}^{-1}$ makes the role of the geometric step explicit: larger values of $\textnormal{sf}$ lead to larger relative stochastic fluctuations between the two compared ensemble sizes.

We now pass from the signed relative difference to the absolute relative error used in the plateau criterion. 
The fluctuation order remains the same, whereas the leading variance constant changes under the absolute-value transformation.

\begin{proposition}
\label{prop:absolute_variance}
Assume that the conditional mean  trajectory
$\mu_T=\mathbb{E}[S_T|D]$
satisfies
\begin{equation*}
\mu_T = S_\infty + cT^{-\gamma} + o(T^{-\gamma}),
\qquad T\to\infty, \quad c\neq 0,
\end{equation*}
with $\gamma>1/2$, that \eqref{eq:remark} holds, and that the signed relative error $(S_R-S_B)/S_B$, conditionally on $D$, is approximately Gaussian. Then
\begin{equation}
\mathrm{Var}\left[\left|\frac{S_R-S_B}{S_B}\right|\middle| D\right]
\sim
\left(1-\frac{2}{\pi}\right)
\frac{v}{S_\infty^2}\cdot\frac{1-\textnormal{sf}^{-1}}{B},
\qquad B\to\infty.
\label{eq:prop3_main}
\end{equation}
\end{proposition}
The derivation of the above asymptotic formula also relies on the same tail model \eqref{eq:score_tail_alpha}
for the conditional mean trajectory, together with the asymptotic behavior of the conditional mean
$\mathbb{E}[(S_R-S_B)/S_B\mid D]$
provided by Lemma~\ref{lemma:delta}.
The condition $\gamma>1/2$
is needed only for the derivation: the formula itself does not involve $\gamma$, and passing to the absolute value merely reduces the leading variance constant by the factor $1-2/\pi<1$.
Hence, in the main case of interest, $\gamma=1$, the variance of the plateau statistic $\textnormal{plat}_R$ in \eqref{plat_L_R} is asymptotically given by \eqref{eq:prop3_main}.

\subsection{Tolerance Parameter Selection}
When working with classification metrics such as accuracy and ROC-AUC, special attention must be paid to the choice of the tolerance parameter $\varepsilon$ in the plateau condition \eqref{plat_L_R}. 
For accuracy, the minimal meaningful change is approximately $1/n$, corresponding to a change of one correctly classified sample. 
Therefore, a reasonable lower bound for the tolerance relative to a baseline accuracy $\text{acc}_\text{min}$ (e.g. $0.6$) is:
\begin{equation}
\varepsilon_\text{acc} \gtrsim \frac{1}{n \cdot \text{acc}_\text{min}}.
\label{eq:eps_acc}
\end{equation}

For ROC-AUC, the empirical score is derived from the ranking of positive and negative samples. 
A fundamental change in this ranking corresponds to altering the order of one positive-negative pair, which changes the ROC-AUC by $\pm 1/(n_{+} n_{-})$, 
where $n_{+}$ and $n_{-}$ are the numbers of positive and negative samples ($n_{+}+n_{-}=n$).
Therefore, this value serves as a natural scale for meaningful differences. 
Assuming a baseline $\text{auc}_{\min}$, a practical reference for the relative tolerance is:
\begin{equation}
\varepsilon_{\text{auc}} \gtrsim \frac{1}{n_+ n_- \cdot \text{auc}_{\min}}.
\label{eq:eps_auc}
\end{equation}
However, although any nonzero change in ROC-AUC is discrete (an integer multiple of $1/(n_{+}n_{-})$), the net change can still be zero due to cancellations across multiple pairwise swaps.

We now extend this reasoning to multi-class classification with $K>2$ classes, considering two common averaging schemes: One-vs-Rest (OvR) and One-vs-One (OvO). 
In macro-averaged OvR, $K$ binary classifiers are evaluated, and their metrics are averaged. 
The reference quantization scale for the tolerance becomes
\begin{equation}
\varepsilon_{\text{auc-ovr}} \gtrsim 
\frac{1}{K \cdot \text{auc}_{\min}}
\min_{i} \frac{1}{n_i(n - n_i)},
\label{eq:eps_auc_ovr}
\end{equation}
where $n_i$ is the number of samples in class $i$. 
In the OvO scheme, all $K(K-1)/2$ pairwise class comparisons are averaged. 
A minimal change in ranking affects exactly one such pair, leading to the guideline:
\begin{equation}
\varepsilon_{\text{auc-ovo}} \gtrsim 
\frac{2}{K(K-1) \cdot \text{auc}_{\min}}
\min_{i < j} \frac{1}{n_i n_j}.
\label{eq:eps_auc_ovo}
\end{equation}
For a more robust algorithmic behavior --- i.e., to prevent unjustified growth of the tree count and ensure recurrent left shifts --- one may use ``max'' instead of ``min'' in \eqref{eq:eps_auc_ovr} and \eqref{eq:eps_auc_ovo}.
This yields a larger tolerance, making the plateau condition more conservative and less sensitive to random fluctuations and preventing excessive growth, at the risk of potentially stopping before the performance has stabilized enough.

As evident from \eqref{eq:eps_auc_ovr} and \eqref{eq:eps_auc_ovo}, increasing $K$ can reduce these practical scales compared to the binary case \eqref{eq:eps_auc}, whereas decreasing individual class counts $n_i$ has the opposite effect. 
In the balanced case where $n_i = n/K$ these scales simplify to:
\begin{align}
\varepsilon_{\text{auc-ovr}} &\gtrsim \frac{K}{(K-1) n^{2} \cdot \text{auc}_{\min}}, 
\label{eq:eps_auc_ovr_balanced}\\
\varepsilon_{\text{auc-ovo}} &\gtrsim \frac{2K}{(K-1) n^{2} \cdot \text{auc}_{\min}}.
\label{eq:eps_auc_ovo_balanced}
\end{align}
The expression \eqref{eq:eps_auc_ovo_balanced} coincides with \eqref{eq:eps_auc} for $K=2$ and becomes smaller for $K>2$.

As derivations \eqref{eq:eps_acc} -- \eqref{eq:eps_auc_ovo_balanced} show, the key practical guideline is to avoid setting the tolerance $\varepsilon$ too small. 
In general, this is more critical for accuracy than for ROC-AUC, since $\varepsilon_{\text{acc}} \gtrsim O(1/n)$ as $n \to \infty$, whereas for ROC-AUC under class balance we have $\varepsilon_{\text{auc}} \gtrsim O(1/n^2)$. 
It is also noteworthy that these theoretical scales depend only on the sample size $n$ and not on the number of features $p$. Consequently, special care is required when applying the plateau criterion to datasets with small $n$.

\subsection{Implementation Details}
\label{subsec:implementation_details}
The proposed open-source library \texttt{rf\_plateau\_hpo} is implemented in Python and builds upon two established foundations:
the \texttt{RandomForestClassifier} and \texttt{RandomForestRegressor} classes from scikit-learn for predictive model construction and the Optuna framework for hyperparameter optimization.
It provides three tuning functions:
\begin{enumerate}
    \item \texttt{tune\_rf\_oob()}: classic TPE-based tuning with a user-defined range for the number of trees;
    \item \texttt{tune\_rf\_oob\_bohb()}: a BOHB-like multi-fidelity HPO baseline where the number of trees is used as the budget and is allocated according to a fixed geometric ladder;  
    \item \texttt{tune\_rf\_oob\_plateau()}: the plateau-driven approach proposed in this paper.
\end{enumerate}

All three functions expose a common interface for tuning the main structural hyperparameters of Random Forest.
The user can specify the following search spaces:
\begin{itemize}
    \item \param{max\_features\_grid}: values for the number of features considered at each split (default: $\sqrt{p}$, $0.25$, $1/3$, $0.5$, $0.7$, $1.0$);
    \item \param{max\_depth\_range}: tree depth bounds (default: $1$; $40$);
    \item \param{min\_samples\_leaf\_range}: minimum leaf size bounds (default: $1$; $20$);
    \item \param{min\_samples\_split\_range}: minimum split size bounds (default: $2$; $40$);
    \item \param{tune\_criterion}: whether to tune the split criterion; if ``True'', choose among
    ``gini'', ``entropy'', ``log\_loss'' (classification) or ``squared\_error'', ``absolute\_error'' (regression); otherwise use a fixed criterion (default: ``True'').
\end{itemize}
The experiments below evaluate all three methods under this common setup.

The key difference is how the functions handle the number of trees, \param{n\_estimators}.
The classic \texttt{tune\_rf\_oob()} requires an explicit search interval \param{n\_estimators\_range} (default: 50; 2000).
In contrast, the function \texttt{tune\_rf\_oob\_plateau()} replaces it with an adaptive search controlled by three parameters and a safety cap:
\begin{itemize}
    \item \param{n\_estimators\_start} ($T_0$): initial central value of the triplet (default: 100);
    \item \param{scale\_factor} ($\textnormal{sf}$): multiplicative step for shifting the triplet (default: 1.5);
    \item \param{delta} ($\varepsilon$): relative tolerance in the plateau condition (default: $10^{-3}$);
    \item \param{max\_trees}: hard cap on the tree count (default: 10000).
\end{itemize}

The function \texttt{tune\_rf\_oob\_bohb()} is introduced as a baseline for comparison with the proposed plateau-based method.
In each trial, the number of trees is treated as a computational budget and is increased sequentially along a fixed, finite geometric ladder \texttt{n\_estimators\_ladder}, whereas all remaining hyperparameters are sampled using the TPE strategy.
At each rung, the function evaluates the OOB score.
It then applies Hyperband-style pruning (Successive Halving) with parameter $\eta=\texttt{hyperband\_reduction\_factor}$ (default: 3), promoting only a subset of configurations to the next (more expensive) tree budget level.

Internally, the plateau method is implemented as a custom Optuna objective.
For each trial (i.e., each sampled hyperparameter set \param{max\_features}, \param{max\_depth}, etc.), it trains forests of size $L$, $B$, and $R$ using \param{warm\_start}=``True'', evaluates the OOB scores, checks $\text{plat}_L$ and $\text{plat}_R$, and shifts the triplet accordingly.
If the search fails to reach a plateau within the allowed tree budget, the trial is treated as inconclusive and is pruned by the plateau-specific rule.

For this plateau-specific pruning, the implementation raises \texttt{optuna.TrialPruned()} before returning an objective value and without calling \texttt{trial.report(...)} for the baseline OOB score $S_B$.
Thus, no-plateau trials are recorded only as pruned trials with no objective value and no OOB intermediate values available to the sampler.
Only trials that confirm a right-side plateau return $S_B$ as the objective value.
All experiments were run with Optuna~3.5.1.
This plateau-specific no-plateau rule is separate from standard Optuna pruning, which may be applied only after a plateau has already been confirmed and an intermediate value has been reported.
All functions support arbitrary user-provided OOB metrics (e.g., accuracy, ROC-AUC) and integrate with Optuna's sampling and pruning.

All functions return the fitted Random Forest model and the Optuna \param{study} object.
In addition, the function \texttt{tune\_rf\_oob\_plateau()} returns the selected tree count \param{best\_n\_estimators} (i.e., $T$) and a flag indicating whether any trial successfully reached a plateau.
The selected $T$ is also stored in the trial metadata (user attributes) for downstream analysis.
In practice, \param{n\_trials} and \param{max\_trees} should be large enough (i) to let the initial triplet shift into the plateau region and (ii) to explore other hyperparameters once $T$ has stabilized.
If no plateau is reached, the flag is set to ``False'', and the user should increase \param{n\_trials}/\param{max\_trees} or adjust the starting parameters.

\section{Results}
\label{sec:results}
\subsection{Datasets}
\begin{table*}[t]
\centering
\caption{Characteristics of the benchmark datasets, the score evaluation metrics, and the minimal plateau tolerance ($\varepsilon_{\min}$) used for tuning.}
\label{tab:datasets}
\begin{tabular}{ccccccc}
\toprule
\textbf{dataset} & 
\makecell*{\textbf{size} $n\times p$ \\ \textbf{(after preprocessing)}} & 
\textbf{problem} & 
\makecell*{\textbf{number} \\ \textbf{of classes}} & 
\textbf{imbalance} & 
\textbf{metric} &  
$\mathbf{\varepsilon_{\min}}$\\
\midrule
Iris & $150 \times 4$ & classification & 3 & NO & ROC-AUC & $1\times 10^{-3}$ \\
Wine & $178 \times 13$ & classification & 3 & NO & ROC-AUC & $1\times 10^{-3}$ \\
BreastCancer & $569 \times 30$ & classification & 2 & NO & ROC-AUC & $1\times 10^{-3}$ \\
CreditCardDefault & $30000 \times 23$ & classification & 2 & NO & ROC-AUC & $1 \times 10^{-3}$ \\
GiveMeSomeCredit & $120269 \times 10$ & classification & 2 & YES & ROC-AUC & $1 \times 10^{-3}$ \\
Titanic & $891 \times 11$ & classification & 2 & NO & Accuracy & $3\times 10^{-3}$ \\
Diabetes & $442 \times 10$ & regression & -- & -- & RMSE & $1 \times 10^{-3}$ \\
Abalone & $4177 \times 8$ & regression & -- & -- & RMSE & $1 \times 10^{-3}$ \\
CaliforniaHousing & $20640 \times 8$ & regression & -- & -- & RMSE & $1 \times 10^{-3}$ \\
GeneExpressionCancer RNA-Seq &
$801 \times 20531$ &
classification & 5 & YES & Accuracy & $3 \times 10^{-3}$ \\
Arcene &
$200 \times 10000$&
classification & 2 & NO & ROC-AUC & $1 \times 10^{-3}$ \\
Dorothea &
$1150 \times 100000$ &
classification & 2 & YES & ROC-AUC & $1 \times 10^{-3}$ \\
\bottomrule
\end{tabular}
\setlength{\tabcolsep}{1pt}
\end{table*}
We conducted a comprehensive comparative analysis of the proposed triplet-based plateau search algorithm for Random Forest hyperparameter tuning. 
The method was benchmarked against a standard TPE-based search and against a naive early-stopping baseline under classical experimental conditions.
In the main TPE--PLATEAU comparison, we additionally investigated how different configuration options affect performance.
We further compared the plateau-driven pruning mechanism with multi-fidelity Hyperband-style pruning, where the number of trees serves as the resource (budget) along a predefined ladder. 
In addition, we evaluated joint optimization of the tree budget together with the remaining hyperparameters versus a two-stage (decoupled) strategy. 
Finally, we performed sensitivity analyses with respect to the scale factor $\textnormal{sf}$ and the tolerance parameter $\varepsilon$, highlighting their impact on predictive performance and computational cost.

The experiments utilized 12 datasets sourced from the UCI Machine Learning Repository (Kelly et al.~\cite{Kelly2023Uci}), Kaggle~\cite{Kaggle2025Kaggle},
and scikit-learn built-in datasets. 
This collection includes 9 classification and 3 regression tasks.
The datasets vary in size, ranging from $n=150$ samples in the \dataset{Iris} dataset to $n=120,269$ samples in \dataset{GiveMeSomeCredit}. 
The collection also includes low‑density bioinformatics datasets: \dataset{GeneExpressionCancer RNA‑Seq}, \dataset{Arcene}, and \dataset{Dorothea}.
In the \dataset{GiveMeSomeCredit} dataset, samples containing missing values were removed. 
For the \dataset{Titanic} dataset, one-hot encoding was applied to categorical features, and complex string-based attributes were parsed into several newly constructed features.

For classification tasks with imbalanced class distributions, non-uniform class weighting was applied via the \param{class\_weight} parameter.
This is a standard option in scikit-learn's \texttt{RandomForestClassifier} that incorporates class weights into the split criterion (Gini impurity or entropy) minimized during the growth of each tree, without affecting the evaluation metric itself. 
Regarding the choice of evaluation metric for defining the OOB score, we primarily used ROC-AUC for classification and root mean squared error (RMSE) for regression. 
For the \dataset{Titanic} dataset --- a binary classification task with balanced classes --- we used the accuracy metric as an illustrative example, following the convention established in the corresponding Kaggle competition.

In selecting $\varepsilon$, we follow a natural heuristic that an analyst might adopt: to make the ensemble large enough so that the fluctuations in the score do not exceed 1 percent --- i.e., $\varepsilon = 10^{-2}$ --- and, for greater reliability, use values down to $\varepsilon = 10^{-3}$.  
All tolerance values between these two can be considered acceptable when training a Random Forest.
However, the considerations outlined in \eqref{eq:eps_acc}--\eqref{eq:eps_auc_ovo_balanced} should be taken into account.

In our experiments we use values starting from $\varepsilon_{\min} = 10^{-3}$, and then $3\times 10^{-3}$, $5\times 10^{-3}$, $7\times 10^{-3}$, and $9\times 10^{-3}$.
For the \dataset{Titanic} and \dataset{GeneExpressionCancer RNA-Seq} datasets, however, taking into account \eqref{eq:eps_acc}, we start from $\varepsilon_{\min} = 3\times 10^{-3}$.

For regression datasets we set $\varepsilon_{\min} = 10^{-3}$ (see Table \ref{tab:datasets}).  
For the multiclass datasets (\dataset{Iris} and \dataset{Wine}), we evaluate the OOB score using the macro-averaged One-vs-Rest ROC-AUC, since OvR provides a straightforward per-class interpretation and macro-averaging assigns equal weight to all classes, which is natural for (approximately) balanced class sizes.
In the \dataset{Iris} dataset, all three classes have the same size.
The scale \eqref{eq:eps_auc_ovr_balanced} yields a value of about $10^{-4}$; therefore, we keep $\varepsilon_{\min} = 10^{-3}$.  
For the \dataset{Wine} dataset, the scale given by \eqref{eq:eps_auc_ovr} is also about $10^{-4}$ (even when using $\max$ instead of $\min$), so $10^{-3}$ is likewise appropriate.  
The datasets \dataset{BreastCancer}, \dataset{CreditCardDefault}, \dataset{GiveMeSomeCredit},
\dataset{Arcene},
and \dataset{Dorothea} have a sufficiently large $n$ to justify this tolerance.  
\subsection{Experiment Design}
\label{subsec:exp_design}
Our experimental protocol is guided by the following research questions:
\begin{itemize}
    \item How does the proposed plateau-search method differ from the classic TPE hyperparameter search in terms of runtime, the selected number of trees $T$, and sensitivity to configuration options?
    \item How does the number of optimization trials $n_{\text{trials}}$ affect the best achieved score after tuning, and what values of $n_{\text{trials}}$ are sufficient in practice, especially given that some trials may be pruned when a plateau is not reached?
    \item What is the impact of allowing the split criterion to be tuned?
    \item What is the impact of tuning additional hyperparameters beyond tree depth, namely \param{max\_features}, \param{min\_samples\_split}, and \param{min\_samples\_leaf}?
    \item How critical is joint hyperparameter optimization in Random Forest --- and in particular for the proposed plateau-search --- compared to a decoupled (two-stage) strategy where the tree budget is tuned separately from the remaining hyperparameters?
    \item How sensitive is the proposed plateau-based method to the choice of the scale factor \textnormal{sf} and the tolerance parameter $\varepsilon$ in terms of predictive performance and computational cost?
\end{itemize}
We consider the following experimental cases:
\begin{figure*}[t]
\centering
\includegraphics[width=\linewidth]{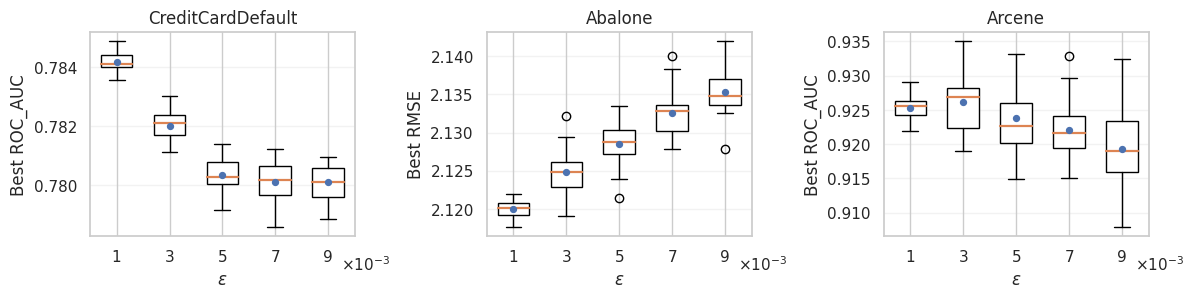}
\caption{Sensitivity of the best achieved OOB performance to the tolerance parameter $\varepsilon$ for three representative datasets: a large-$n$ classification task (\dataset{CreditCardDefault}, ROC-AUC), a regression task (\dataset{Abalone}, RMSE), and a high-dimensional classification task (\dataset{Arcene}, ROC-AUC).
Boxplots summarize 20 random seeds for each $\varepsilon$; the central line denotes the median and boxes indicate the interquartile range (whiskers follow the standard boxplot convention).
Experiments use \textsc{tune\_criterion}=``No'', \textsc{only\_depth}=``No'', $\text{sf}=1.5$, and $n_{\text{trials}}=120$.
}
\label{fig:deltas}  
\end{figure*}
\begin{itemize}
    \item \textbf{TPE}: classic TPE-based joint optimization of all Random Forest hyperparameters, where \param{n\_estimators} is sampled from the fixed range $[T_{\min}, T_{\max}]$ with $T_{\min}=100$ and $T_{\max}=2565$.
    \item \textbf{TPE\_Tmin--Tmax}: a decoupled counterpart of TPE. First, TPE tunes the remaining Random Forest hyperparameters with \param{n\_estimators} fixed at $T_{\min}=100$. Second, a Random Forest is trained using the obtained hyperparameters at \param{n\_estimators}=$T_{\max}=2565$.
    \item \textbf{HB}: a multi-fidelity Hyperband-style baseline that treats \param{n\_estimators} as the resource and evaluates each configuration on the fixed geometric ladder
    $[100, 150, 225, 338, 507, 760, 1140, 1710, 2565]$,
    using $\eta=\texttt{hyperband\_reduction\_factor}=3$.
    \item \textbf{ES}: a naive monotone early-stopping baseline. At each trial, it incrementally traverses the geometric ladder $T_{\min}\,\textnormal{sf}^{\,j}$, $j=0,1,\ldots$, and stops at the first ladder point for which the left relative plateau condition $\text{plat}_L \le \varepsilon$ is satisfied. 
    Unlike fixed-range baselines, it does not use an explicit upper bound $T_{\max}$; instead, the ladder is truncated only by a large safety cap \param{max\_trees}.
    \item \textbf{PLATEAU}: the proposed triplet-based plateau search with $T_0=100$, where the central point of the triplet follows the geometric progression $B \approx T_0\,\textnormal{sf}^{\,j}$, $j=0,1,\ldots$. In contrast to ES, the search does not restart from the leftmost ladder point at every trial, but adaptively updates the working triplet across trials. As in ES, no explicit upper bound $T_{\max}$ is used; instead, a large safety cap \param{max\_trees} is imposed.
    \item \textbf{TPE\_Tmin--PLT}: a decoupled counterpart of PLATEAU. 
    First, TPE tunes the remaining Random Forest hyperparameters with \param{n\_estimators} fixed at $T_{\min}=100$. Second, the plateau procedure is applied under this fixed hyperparameter configuration to adaptively select the number of trees.
\end{itemize}
The Hyperband ladder and the upper bound $T_{\max}=2565$ are derived from the same geometric progression with starting point  $T_0=T_{\min}=100$ and scale factor $\textnormal{sf}=1.5$, and we set $T_{\max}$ to the first ladder value that exceeds $2000$.

In our experimental infrastructure, a single hyperparameter-tuning run for Random Forest is specified by the following settings:
\begin{itemize}
    \item dataset,
    \item \param{tune\_criterion} $\in$ \{``Yes'', ``No''\},
    \item \param{only\_depth} $\in$ \{``Yes'', ``No''\}, 
    \item algorithm $\in$ \{TPE, TPE\_Tmin--Tmax,
    HB, ES, PLATEAU, TPE\_Tmin--PLT\},
    \item number of trials $n_{\text{trials}}\in \{40, 120\}$,
    \item random seed (20 distinct values).
\end{itemize}
For the PLATEAU algorithm, we additionally consider $\textnormal{sf}\in\{1.5,\,2\}$ and $\varepsilon \in \{10^{-3},\, 3\times 10^{-3},\, 5\times 10^{-3},\, 7\times 10^{-3}, 9\times 10^{-3}\}$.

Unless stated otherwise, each run used one physical core (2 vCPUs) of an Intel Xeon Gold 6230 running at 2.10 GHz base frequency; for \dataset{Dorothea}, each run used two physical cores (4 vCPUs) due to its substantially higher computational cost.
For reproducibility, the 20 repeated runs for each configuration were generated using an explicit fixed-seed protocol, so that the same sequence of random seeds can be reproduced across reruns.
Unless stated otherwise, each run of the plateau-search algorithm includes the revisit phase.
\begin{figure*}[t]
\centering
\includegraphics[width=\linewidth]{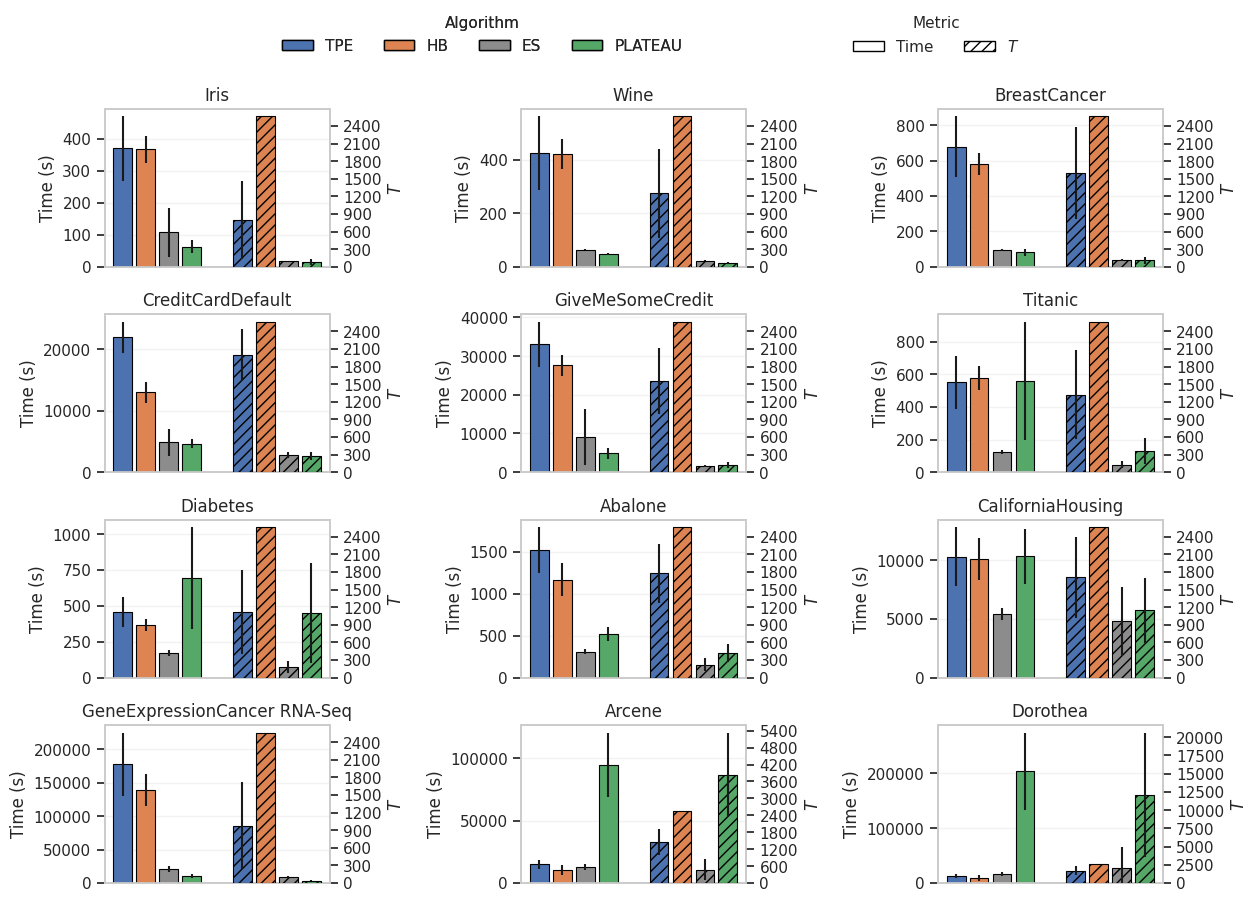}
\caption{Runtime and selected number of trees $T$ for TPE, HB, ES, and the proposed PLATEAU search across all datasets.
Bars show the mean over 20 random seeds and error bars indicate $\pm1$ standard deviation.
For TPE and HB, the tree-count search uses the fixed upper bound $T_{\max}=2565$; for ES and PLATEAU, the stopping tolerance is set to $\varepsilon=\varepsilon_{\min}$.
All experiments use \textsc{tune\_criterion}=``No'', \textsc{only\_depth}=``No'', $\text{sf}=1.5$, and $n_{\text{trials}}=120$.
\dataset{Dorothea} was run with two physical cores per run instead of one because of its higher computational cost.}
\label{fig:bars}  
\end{figure*}
\begin{table*}[t]
\centering
\caption{
Effect of the trial budget ($n_{\text{trials}}=120$ vs.\ $40$) on the best achieved OOB score, with $\text{sf}=1.5$ and $\varepsilon=\varepsilon_{\min}$.
For each dataset and configuration (algorithm, \textsc{tune\_criterion} =``Yes/No'', \textsc{only\_depth}=``Yes/No''), we compare two groups of 20 runs with different random seeds.
Each cell reports the one-sided $p$-value and the corresponding effect size (Cohen's $d$ or Cliff's $\delta$, depending on normality).
Boldface indicates statistically significant and practically non-negligible differences.
}
\label{tab:n_trials_120_vs_40}
\setlength{\tabcolsep}{1pt}
\begin{tabular}{>{\centering\arraybackslash}m{2.8cm}|>{\centering\arraybackslash}m{1.6cm}|>{\centering\arraybackslash}m{1.6cm}|>{\centering\arraybackslash}m{1.6cm}|>{\centering\arraybackslash}m{1.6cm}|>{\centering\arraybackslash}m{1.6cm}|>{\centering\arraybackslash}m{1.6cm}|>{\centering\arraybackslash}m{1.6cm}|>{\centering\arraybackslash}m{1.6cm}}
\toprule
$\mathbf{n_{trials}:}$ & \multicolumn{8}{c}{\textbf{120 vs 40}} \\
\midrule
\textbf{algorithm:} & \multicolumn{4}{c|}{\textbf{TPE}} & \multicolumn{4}{c}{\textbf{PLATEAU}} \\
\midrule
\textbf{tune criterion:} & \multicolumn{2}{c|}{\textbf{YES}} & \multicolumn{2}{c|}{\textbf{NO}} & \multicolumn{2}{c|}{\textbf{YES}} & \multicolumn{2}{c}{\textbf{NO}} \\
\midrule
\textbf{only depth:} & \textbf{YES} & \textbf{NO} & \textbf{YES} & \textbf{NO} & \textbf{YES} & \textbf{NO} & \textbf{YES} & \textbf{NO} \\
\midrule
\makecell*{Iris} & \makecell*{$\mathbf{8.5\times10^{-4}}$\\$\delta=\mathbf{0.58}$} & \makecell*{$\mathbf{1.2\times10^{-3}}$\\$\delta=\mathbf{0.56}$} & \makecell*{$\mathbf{1.2\times10^{-3}}$\\$\delta=\mathbf{0.56}$} & \makecell*{$\mathbf{4.1\times10^{-4}}$\\$d=\mathbf{1.19}$} & \makecell*{$\mathbf{6.0\times10^{-3}}$\\$d=\mathbf{0.88}$} & \makecell*{$\mathbf{1.3\times10^{-4}}$\\$d=\mathbf{1.33}$} & \makecell*{$\mathbf{2.0\times10^{-3}}$\\$d=\mathbf{1.00}$} & \makecell*{$\mathbf{4.4\times10^{-4}}$\\$d=\mathbf{1.19}$} \\
\makecell*{Wine} & \makecell*{$\mathbf{8.0\times10^{-5}}$\\$\delta=\mathbf{0.66}$} & \makecell*{$\mathbf{3.5\times10^{-6}}$\\$\delta=\mathbf{0.83}$} & \makecell*{$\mathbf{2.0\times10^{-2}}$\\$\delta=\mathbf{0.38}$} & \makecell*{$\mathbf{8.8\times10^{-6}}$\\$d=\mathbf{1.61}$} & \makecell*{$\mathbf{1.2\times10^{-2}}$\\$d=\mathbf{0.79}$} & \makecell*{$\mathbf{3.4\times10^{-4}}$\\$\delta=\mathbf{0.63}$} & \makecell*{$\mathbf{6.7\times10^{-5}}$\\$\delta=\mathbf{0.71}$} & \makecell*{$\mathbf{5.0\times10^{-7}}$\\$d=\mathbf{1.96}$} \\
\makecell*{BreastCancer} & \makecell*{$\mathbf{1.0\times10^{-3}}$\\$\delta=\mathbf{0.57}$} & \makecell*{$\mathbf{1.6\times10^{-6}}$\\$d=\mathbf{1.95}$} & \makecell*{$\mathbf{1.8\times10^{-5}}$\\$d=\mathbf{1.54}$} & \makecell*{$\mathbf{3.2\times10^{-7}}$\\$d=\mathbf{2.05}$} & \makecell*{$\mathbf{9.3\times10^{-4}}$\\$\delta=\mathbf{0.58}$} & \makecell*{$\mathbf{4.5\times10^{-5}}$\\$d=\mathbf{1.44}$} & \makecell*{$\mathbf{1.0\times10^{-3}}$\\$d=\mathbf{1.13}$} & \makecell*{$\mathbf{1.0\times10^{-5}}$\\$d=\mathbf{1.63}$} \\
\makecell*{CreditCardDefault} & \makecell*{$\mathbf{1.9\times10^{-4}}$\\$\delta=\mathbf{0.66}$} & \makecell*{$\mathbf{5.9\times10^{-6}}$\\$d=\mathbf{1.71}$} & \makecell*{$\mathbf{6.1\times10^{-4}}$\\$d=\mathbf{1.21}$} & \makecell*{$\mathbf{3.9\times10^{-5}}$\\$d=\mathbf{1.58}$} & \makecell*{$\mathbf{6.8\times10^{-6}}$\\$d=\mathbf{1.69}$} & \makecell*{$\mathbf{7.6\times10^{-6}}$\\$d=\mathbf{1.68}$} & \makecell*{$\mathbf{2.8\times10^{-6}}$\\$d=\mathbf{1.74}$} & \makecell*{$\mathbf{4.4\times10^{-4}}$\\$d=\mathbf{1.19}$} \\
\makecell*{GiveMeSomeCredit} & \makecell*{$\mathbf{2.3\times10^{-6}}$\\$\delta=\mathbf{0.85}$} & \makecell*{$\mathbf{4.9\times10^{-6}}$\\$\delta=\mathbf{0.82}$} & \makecell*{$\mathbf{1.8\times10^{-7}}$\\$d=\mathbf{2.03}$} & \makecell*{$\mathbf{3.7\times10^{-5}}$\\$\delta=\mathbf{0.74}$} & \makecell*{$\mathbf{3.0\times10^{-2}}$\\$\delta=\mathbf{0.35}$} & \makecell*{$\mathbf{4.8\times10^{-4}}$\\$\delta=\mathbf{0.61}$} & \makecell*{$\mathbf{8.3\times10^{-3}}$\\$\delta=\mathbf{0.45}$} & \makecell*{$\mathbf{4.5\times10^{-3}}$\\$d=\mathbf{0.93}$} \\
\makecell*{Titanic} & \makecell*{$\mathbf{4.2\times10^{-4}}$\\$d=\mathbf{1.18}$} & \makecell*{$\mathbf{8.6\times10^{-7}}$\\$\delta=\mathbf{0.88}$} & \makecell*{$\mathbf{1.0\times10^{-4}}$\\$\delta=\mathbf{0.68}$} & \makecell*{$\mathbf{3.5\times10^{-5}}$\\$\delta=\mathbf{0.73}$} & \makecell*{$\mathbf{1.3\times10^{-3}}$\\$\delta=\mathbf{0.55}$} & \makecell*{$\mathbf{2.8\times10^{-6}}$\\$\delta=\mathbf{0.84}$} & \makecell*{$1.1\times10^{-1}$\\$\delta=0.23$} & \makecell*{$\mathbf{4.4\times10^{-4}}$\\$\delta=\mathbf{0.61}$} \\
\makecell*{Diabetes} & \makecell*{$\mathbf{3.7\times10^{-5}}$\\$\delta=\mathbf{-0.73}$} & \makecell*{$\mathbf{3.1\times10^{-5}}$\\$d=\mathbf{-1.51}$} & \makecell*{$\mathbf{2.2\times10^{-5}}$\\$\delta=\mathbf{-0.76}$} & \makecell*{$\mathbf{5.1\times10^{-5}}$\\$d=\mathbf{-1.46}$} & \makecell*{$\mathbf{1.7\times10^{-3}}$\\$\delta=\mathbf{-0.54}$} & \makecell*{$\mathbf{7.8\times10^{-4}}$\\$d=\mathbf{-1.19}$} & \makecell*{$\mathbf{8.7\times10^{-6}}$\\$d=\mathbf{-1.72}$} & \makecell*{$\mathbf{1.4\times10^{-3}}$\\$\delta=\mathbf{-0.56}$} \\
\makecell*{Abalone} & \makecell*{$\mathbf{8.1\times10^{-6}}$\\$d=\mathbf{-1.65}$} & \makecell*{$\mathbf{1.0\times10^{-6}}$\\$d=\mathbf{-1.91}$} & \makecell*{$\mathbf{3.7\times10^{-4}}$\\$d=\mathbf{-1.24}$} & \makecell*{$\mathbf{6.3\times10^{-7}}$\\$d=\mathbf{-1.99}$} & \makecell*{$\mathbf{2.3\times10^{-6}}$\\$\delta=\mathbf{-0.85}$} & \makecell*{$\mathbf{2.2\times10^{-6}}$\\$d=\mathbf{-1.96}$} & \makecell*{$\mathbf{8.0\times10^{-5}}$\\$\delta=\mathbf{-0.70}$} & \makecell*{$\mathbf{2.9\times10^{-6}}$\\$d=\mathbf{-1.79}$} \\
\makecell*{CaliforniaHousing} & \makecell*{$\mathbf{1.7\times10^{-5}}$\\$d=\mathbf{-1.62}$} & \makecell*{$\mathbf{4.9\times10^{-7}}$\\$\delta=\mathbf{-0.91}$} & \makecell*{$\mathbf{9.3\times10^{-5}}$\\$d=\mathbf{-1.42}$} & \makecell*{$\mathbf{4.3\times10^{-6}}$\\$\delta=\mathbf{-0.83}$} & \makecell*{$\mathbf{2.3\times10^{-5}}$\\$d=\mathbf{-1.56}$} & \makecell*{$\mathbf{1.3\times10^{-6}}$\\$\delta=\mathbf{-0.87}$} & \makecell*{$\mathbf{1.8\times10^{-4}}$\\$d=\mathbf{-1.34}$} & \makecell*{$\mathbf{2.4\times10^{-4}}$\\$\delta=\mathbf{-0.65}$} \\
\makecell*{GeneExpressionCancer \\ RNA-Seq} & \makecell*{$\mathbf{3.9\times10^{-3}}$\\$\delta=\mathbf{0.44}$} & \makecell*{$\mathbf{5.6\times10^{-4}}$\\$\delta=\mathbf{0.54}$} & \makecell*{$\mathbf{4.3\times10^{-3}}$\\$\delta=\mathbf{0.43}$} & \makecell*{$\mathbf{5.2\times10^{-3}}$\\$\delta=\mathbf{0.43}$} & \makecell*{$\mathbf{1.4\times10^{-2}}$\\$\delta=\mathbf{0.38}$} & \makecell*{$\mathbf{8.8\times10^{-3}}$\\$\delta=\mathbf{0.41}$} & \makecell*{$\mathbf{1.4\times10^{-3}}$\\$\delta=\mathbf{0.51}$} & \makecell*{$\mathbf{3.1\times10^{-3}}$\\$\delta=\mathbf{0.46}$} \\
\makecell*{Arcene} & \makecell*{$\mathbf{3.8\times10^{-4}}$\\$d=\mathbf{1.22}$} & \makecell*{$\mathbf{4.3\times10^{-9}}$\\$d=\mathbf{2.52}$} & \makecell*{$\mathbf{6.1\times10^{-4}}$\\$\delta=\mathbf{0.60}$} & \makecell*{$\mathbf{6.0\times10^{-7}}$\\$\delta=\mathbf{0.90}$} & \makecell*{$\mathbf{2.2\times10^{-2}}$\\$\delta=\mathbf{0.37}$} & \makecell*{$\mathbf{2.0\times10^{-7}}$\\$\delta=\mathbf{0.94}$} & \makecell*{$1.1\times10^{-1}$\\$d=0.42$} & \makecell*{$\mathbf{5.2\times10^{-7}}$\\$\delta=\mathbf{0.90}$} \\
\makecell*{Dorothea} & \makecell*{$\mathbf{3.3\times10^{-5}}$\\$d=\mathbf{1.50}$} & \makecell*{$\mathbf{2.2\times10^{-6}}$\\$d=\mathbf{1.78}$} & \makecell*{$\mathbf{6.4\times10^{-3}}$\\$d=\mathbf{0.86}$} & \makecell*{$\mathbf{2.2\times10^{-9}}$\\$d=\mathbf{2.59}$} & \makecell*{$\mathbf{2.1\times10^{-2}}$\\$d=\mathbf{0.71}$} & \makecell*{$\mathbf{1.2\times10^{-7}}$\\$d=\mathbf{2.11}$} & \makecell*{$\mathbf{1.5\times10^{-2}}$\\$\delta=\mathbf{0.40}$} & \makecell*{$\mathbf{2.7\times10^{-6}}$\\$d=\mathbf{1.81}$} \\
\bottomrule
\end{tabular}

\end{table*}
\begin{table*}[t]
\centering
\caption{Impact of enabling \textsc{tune\_criterion} and restricting the search to \textsc{only\_depth} (``Yes'' vs.\ ``No'') at $\text{sf}=1.5$, $n_{\text{trials}}=120$.
For each dataset and algorithm, we compare two groups of 20 runs with different random seeds using two-sided tests.
Each cell reports the $p$-value and the corresponding effect size (Cohen's $d$ or Cliff's $\delta$).
Boldface indicates statistically significant and practically non-negligible differences.}
\label{tab:criterion_depth_yes_vs_no}
\setlength{\tabcolsep}{1pt}
\begin{tabular}{>{\centering\arraybackslash}m{2.8cm}|>{\centering\arraybackslash}m{1.7cm}|>{\centering\arraybackslash}m{1.7cm}|>{\centering\arraybackslash}m{1.7cm}|>{\centering\arraybackslash}m{1.7cm}|>{\centering\arraybackslash}m{1.7cm}|>{\centering\arraybackslash}m{1.7cm}|>{\centering\arraybackslash}m{1.7cm}|>{\centering\arraybackslash}m{1.7cm}}
\toprule
\textbf{tune criterion:} & \multicolumn{4}{c|}{\textbf{YES vs NO}} & \textbf{YES} & \textbf{NO} & \textbf{YES} & \textbf{NO} \\
\midrule
\textbf{only depth:} & \textbf{YES} & \textbf{NO} & \textbf{YES} & \textbf{NO} & \multicolumn{4}{c}{\textbf{YES vs NO}} \\
\midrule
\textbf{algorithm:} & \multicolumn{2}{c|}{\textbf{TPE}} & \multicolumn{2}{c|}{\textbf{PLATEAU}} & \multicolumn{2}{c|}{\textbf{TPE}} & \multicolumn{2}{c}{\textbf{PLATEAU}} \\
\midrule
\makecell*{Iris} & \makecell*{$5.3\times10^{-1}$\\$\delta=-0.12$} & \makecell*{$2.8\times10^{-1}$\\$\delta=-0.20$} & \makecell*{$1.8\times10^{-1}$\\$d=-0.45$} & \makecell*{$5.2\times10^{-1}$\\$d=-0.21$} & \makecell*{$\mathbf{1.0\times10^{-2}}$\\$\delta=\mathbf{-0.47}$} & \makecell*{$\mathbf{2.3\times10^{-2}}$\\$\delta=\mathbf{-0.42}$} & \makecell*{$\mathbf{1.0\times10^{-2}}$\\$d=\mathbf{-0.88}$} & \makecell*{$8.4\times10^{-2}$\\$d=\mathbf{-0.59}$} \\
\makecell*{Wine} & \makecell*{$\mathbf{4.3\times10^{-5}}$\\$\delta=\mathbf{0.73}$} & \makecell*{$\mathbf{3.0\times10^{-3}}$\\$d=\mathbf{1.08}$} & \makecell*{$5.8\times10^{-1}$\\$\delta=-0.10$} & \makecell*{$6.7\times10^{-2}$\\$\delta=\mathbf{0.34}$} & \makecell*{$\mathbf{1.4\times10^{-7}}$\\$\delta=\mathbf{0.96}$} & \makecell*{$\mathbf{6.5\times10^{-6}}$\\$\delta=\mathbf{0.83}$} & \makecell*{$7.4\times10^{-1}$\\$\delta=0.06$} & \makecell*{$\mathbf{2.6\times10^{-3}}$\\$\delta=\mathbf{0.56}$} \\
\makecell*{BreastCancer} & \makecell*{$\mathbf{1.7\times10^{-7}}$\\$\delta=\mathbf{0.97}$} & \makecell*{$\mathbf{2.3\times10^{-9}}$\\$d=\mathbf{2.63}$} & \makecell*{$\mathbf{5.1\times10^{-4}}$\\$\delta=\mathbf{0.65}$} & \makecell*{$7.1\times10^{-2}$\\$d=\mathbf{0.61}$} & \makecell*{$\mathbf{1.9\times10^{-5}}$\\$\delta=\mathbf{0.79}$} & \makecell*{$\mathbf{1.1\times10^{-7}}$\\$d=\mathbf{2.29}$} & \makecell*{$\mathbf{3.5\times10^{-6}}$\\$\delta=\mathbf{0.86}$} & \makecell*{$\mathbf{1.9\times10^{-5}}$\\$d=\mathbf{1.68}$} \\
\makecell*{CreditCardDefault} & \makecell*{$\mathbf{9.9\times10^{-29}}$\\$d=\mathbf{13.05}$} & \makecell*{$\mathbf{1.6\times10^{-8}}$\\$d=\mathbf{2.57}$} & \makecell*{$\mathbf{2.9\times10^{-18}}$\\$d=\mathbf{5.27}$} & \makecell*{$\mathbf{4.0\times10^{-3}}$\\$d=\mathbf{1.04}$} & \makecell*{$\mathbf{6.8\times10^{-37}}$\\$d=\mathbf{-17.51}$} & \makecell*{$\mathbf{1.0\times10^{-51}}$\\$d=\mathbf{-44.03}$} & \makecell*{$\mathbf{1.0\times10^{-22}}$\\$d=\mathbf{-7.06}$} & \makecell*{$\mathbf{3.3\times10^{-21}}$\\$d=\mathbf{-7.63}$} \\
\makecell*{GiveMeSomeCredit} & \makecell*{$\mathbf{6.8\times10^{-8}}$\\$\delta=\mathbf{1.00}$} & \makecell*{$2.9\times10^{-1}$\\$\delta=-0.20$} & \makecell*{$\mathbf{9.5\times10^{-8}}$\\$d=\mathbf{2.36}$} & \makecell*{$2.6\times10^{-1}$\\$d=-0.37$} & \makecell*{$\mathbf{6.8\times10^{-8}}$\\$\delta=\mathbf{-1.00}$} & \makecell*{$\mathbf{6.8\times10^{-8}}$\\$\delta=\mathbf{-1.00}$} & \makecell*{$\mathbf{1.0\times10^{-25}}$\\$d=\mathbf{-8.62}$} & \makecell*{$\mathbf{6.2\times10^{-23}}$\\$d=\mathbf{-11.17}$} \\
\makecell*{Titanic} & \makecell*{$6.4\times10^{-1}$\\$\delta=0.09$} & \makecell*{$\mathbf{2.4\times10^{-3}}$\\$\delta=\mathbf{0.55}$} & \makecell*{$3.4\times10^{-1}$\\$\delta=0.17$} & \makecell*{$\mathbf{1.1\times10^{-3}}$\\$\delta=\mathbf{0.60}$} & \makecell*{$\mathbf{4.2\times10^{-2}}$\\$\delta=\mathbf{-0.37}$} & \makecell*{$4.2\times10^{-1}$\\$\delta=0.15$} & \makecell*{$\mathbf{2.5\times10^{-3}}$\\$\delta=\mathbf{-0.55}$} & \makecell*{$6.9\times10^{-1}$\\$\delta=0.07$} \\
\makecell*{Diabetes} & \makecell*{$2.3\times10^{-1}$\\$\delta=0.23$} & \makecell*{$7.4\times10^{-1}$\\$d=-0.12$} & \makecell*{$6.9\times10^{-1}$\\$\delta=0.08$} & \makecell*{$\mathbf{4.4\times10^{-2}}$\\$\delta=\mathbf{0.37}$} & \makecell*{$\mathbf{9.6\times10^{-12}}$\\$d=\mathbf{3.22}$} & \makecell*{$\mathbf{1.0\times10^{-6}}$\\$\delta=\mathbf{0.91}$} & \makecell*{$\mathbf{6.8\times10^{-8}}$\\$\delta=\mathbf{1.00}$} & \makecell*{$\mathbf{6.8\times10^{-8}}$\\$\delta=\mathbf{1.00}$} \\
\makecell*{Abalone} & \makecell*{$6.4\times10^{-1}$\\$d=0.15$} & \makecell*{$1.2\times10^{-1}$\\$d=\mathbf{0.51}$} & \makecell*{$3.9\times10^{-1}$\\$\delta=0.16$} & \makecell*{$2.5\times10^{-1}$\\$d=-0.39$} & \makecell*{$\mathbf{2.4\times10^{-39}}$\\$d=\mathbf{19.95}$} & \makecell*{$\mathbf{9.7\times10^{-35}}$\\$d=\mathbf{19.83}$} & \makecell*{$\mathbf{4.6\times10^{-29}}$\\$d=\mathbf{13.38}$} & \makecell*{$\mathbf{6.8\times10^{-8}}$\\$\delta=\mathbf{1.00}$} \\
\makecell*{CaliforniaHousing} & \makecell*{$7.4\times10^{-1}$\\$d=0.11$} & \makecell*{$\mathbf{2.3\times10^{-3}}$\\$\delta=\mathbf{-0.56}$} & \makecell*{$8.4\times10^{-2}$\\$d=\mathbf{0.58}$} & \makecell*{$\mathbf{5.0\times10^{-2}}$\\$\delta=\mathbf{-0.37}$} & \makecell*{$6.0\times10^{-1}$\\$\delta=0.10$} & \makecell*{$\mathbf{1.4\times10^{-5}}$\\$\delta=\mathbf{-0.81}$} & \makecell*{$\mathbf{1.2\times10^{-2}}$\\$\delta=\mathbf{-0.46}$} & \makecell*{$\mathbf{3.9\times10^{-7}}$\\$\delta=\mathbf{-0.94}$} \\
\makecell*{GeneExpressionCancer \\ RNA-Seq} & \makecell*{$\mathbf{3.6\times10^{-2}}$\\$\delta=\mathbf{-0.33}$} & \makecell*{$9.9\times10^{-2}$\\$\delta=0.27$} & \makecell*{$1.0\times10^{-1}$\\$\delta=-0.27$} & \makecell*{$6.7\times10^{-1}$\\$\delta=-0.07$} & \makecell*{$\mathbf{2.2\times10^{-2}}$\\$\delta=\mathbf{-0.37}$} & \makecell*{$1.5\times10^{-1}$\\$\delta=0.24$} & \makecell*{$2.5\times10^{-1}$\\$\delta=-0.19$} & \makecell*{$9.3\times10^{-1}$\\$\delta=0.02$} \\
\makecell*{Arcene} & \makecell*{$8.4\times10^{-1}$\\$d=0.08$} & \makecell*{$9.4\times10^{-1}$\\$d=0.02$} & \makecell*{$4.8\times10^{-1}$\\$d=-0.24$} & \makecell*{$3.2\times10^{-1}$\\$d=-0.33$} & \makecell*{$\mathbf{4.2\times10^{-7}}$\\$d=\mathbf{2.08}$} & \makecell*{$\mathbf{4.1\times10^{-9}}$\\$d=\mathbf{2.52}$} & \makecell*{$\mathbf{1.1\times10^{-9}}$\\$d=\mathbf{2.92}$} & \makecell*{$\mathbf{6.7\times10^{-9}}$\\$d=\mathbf{2.77}$} \\
\makecell*{Dorothea} & \makecell*{$\mathbf{4.1\times10^{-4}}$\\$d=\mathbf{1.28}$} & \makecell*{$\mathbf{2.3\times10^{-2}}$\\$d=\mathbf{0.78}$} & \makecell*{$\mathbf{4.2\times10^{-5}}$\\$\delta=\mathbf{0.76}$} & \makecell*{$4.5\times10^{-1}$\\$d=0.25$} & \makecell*{$\mathbf{6.3\times10^{-6}}$\\$d=\mathbf{1.94}$} & \makecell*{$\mathbf{4.2\times10^{-8}}$\\$d=\mathbf{2.53}$} & \makecell*{$\mathbf{2.3\times10^{-6}}$\\$d=\mathbf{2.14}$} & \makecell*{$\mathbf{5.2\times10^{-6}}$\\$\delta=\mathbf{0.85}$} \\
\bottomrule
\end{tabular}

\end{table*}
\begin{table*}[t]
\centering
\caption{
Supplementary ablation and sensitivity analysis.
The table compares joint vs.\ decoupled tuning (TPE and PLATEAU), pruning aggressiveness (HB vs.\ PLATEAU), the computational cost of HB vs.\ TPE (wall-clock time and total number of trees built during training), and the sensitivity of PLATEAU to the scale factor  $\text{sf}$ (1.5 vs.\ 2.0).
Unless stated otherwise: 
\textsc{tune\_criterion} =``No'', \textsc{only\_depth}=``No'', $\text{sf}=1.5$, $\varepsilon=\varepsilon_{\min}$, $n_{\text{trials}}=120$.
For each dataset and comparison, we evaluate two groups of 20 runs with different random seeds using two-sided tests.
Each cell reports the $p$-value and the corresponding effect size (Cohen's $d$ or Cliff's $\delta$, depending on normality).
Boldface indicates statistically significant and practically non-negligible differences.
}
\label{tab:joint_hb_sf}
\setlength{\tabcolsep}{1pt}
\begin{tabular}{>{\centering\arraybackslash}m{2.7cm}|>{\centering\arraybackslash}m{2.2cm}|>{\centering\arraybackslash}m{2.2cm}|>{\centering\arraybackslash}m{1.8cm}|>{\centering\arraybackslash}m{1.8cm}|>{\centering\arraybackslash}m{1.8cm}|>{\centering\arraybackslash}m{1.8cm}|>{\centering\arraybackslash}m{1.8cm}}
\toprule
\textbf{algorithm:} 
& \makecell{\textbf{score:} \\ \textbf{TPE\_\mbox{Tmin-Tmax}} \\ \textbf{vs TPE}} 
& \makecell{\textbf{score:} \\ \textbf{TPE\_\mbox{Tmin-PLT}} \\ \textbf{vs PLATEAU}}
& \makecell{\textbf{trials pruned:} \\ \textbf{HB vs} \\ \textbf{PLATEAU}} 
& \makecell{\textbf{time:} \\ \textbf{HB vs TPE} \\ {\ }} 
& \makecell{\textbf{trees built:} \\ \textbf{HB vs TPE} \\ {\ }} 
& \makecell{\textbf{time:} \\ $\mathbf{sf=1.5}$ \\ \textbf{vs} $\mathbf{sf=2.0}$} 
& \makecell{\textbf{trees built:} \\ $\mathbf{sf=1.5}$ \\ \textbf{vs} $\mathbf{sf=2.0}$} \\
\midrule
\makecell*{Iris} & \makecell*{$\mathbf{6.0\times10^{-12}}$\\$d=\mathbf{-3.38}$} & \makecell*{$1.9\times10^{-1}$\\$d=0.44$} & \makecell*{$\mathbf{1.3\times10^{-3}}$\\$d=\mathbf{1.17}$} & \makecell*{$9.1\times10^{-1}$\\$d=-0.03$} & \makecell*{$\mathbf{1.7\times10^{-3}}$\\$d=\mathbf{-1.17}$} & \makecell*{$\mathbf{1.4\times10^{-6}}$\\$\delta=\mathbf{-0.89}$} & \makecell*{$\mathbf{5.2\times10^{-7}}$\\$\delta=\mathbf{-0.93}$} \\
\makecell*{Wine} & \makecell*{$\mathbf{2.1\times10^{-8}}$\\$d=\mathbf{-2.72}$} & \makecell*{$5.7\times10^{-1}$\\$d=0.18$} & \makecell*{$9.9\times10^{-2}$\\$d=\mathbf{0.56}$} & \makecell*{$9.0\times10^{-1}$\\$d=-0.03$} & \makecell*{$\mathbf{6.0\times10^{-3}}$\\$d=\mathbf{-1.00}$} & \makecell*{$\mathbf{1.2\times10^{-14}}$\\$d=\mathbf{-4.80}$} & \makecell*{$\mathbf{7.8\times10^{-15}}$\\$d=\mathbf{-5.03}$} \\
\makecell*{BreastCancer} & \makecell*{$\mathbf{5.2\times10^{-10}}$\\$d=\mathbf{-3.03}$} & \makecell*{$5.3\times10^{-2}$\\$d=\mathbf{0.65}$} & \makecell*{$1.1\times10^{-1}$\\$d=\mathbf{0.53}$} & \makecell*{$7.2\times10^{-2}$\\$\delta=\mathbf{-0.33}$} & \makecell*{$\mathbf{4.7\times10^{-5}}$\\$d=\mathbf{-1.70}$} & \makecell*{$\mathbf{1.7\times10^{-7}}$\\$\delta=\mathbf{-0.97}$} & \makecell*{$\mathbf{9.2\times10^{-8}}$\\$\delta=\mathbf{-0.99}$} \\
\makecell*{CreditCardDefault} & \makecell*{$\mathbf{3.0\times10^{-7}}$\\$\delta=\mathbf{-0.95}$} & \makecell*{$8.0\times10^{-1}$\\$\delta=0.05$} & \makecell*{$\mathbf{1.1\times10^{-3}}$\\$d=\mathbf{1.19}$} & \makecell*{$\mathbf{2.0\times10^{-14}}$\\$d=\mathbf{-4.31}$} & \makecell*{$\mathbf{9.2\times10^{-8}}$\\$\delta=\mathbf{-0.99}$} & \makecell*{$\mathbf{6.8\times10^{-8}}$\\$\delta=\mathbf{-1.00}$} & \makecell*{$\mathbf{6.8\times10^{-8}}$\\$\delta=\mathbf{-1.00}$} \\
\makecell*{GiveMeSomeCredit} & \makecell*{$\mathbf{2.0\times10^{-5}}$\\$\delta=\mathbf{-0.79}$} & \makecell*{$\mathbf{1.3\times10^{-2}}$\\$d=\mathbf{-0.85}$} & \makecell*{$\mathbf{1.3\times10^{-2}}$\\$d=\mathbf{-0.86}$} & \makecell*{$\mathbf{1.0\times10^{-3}}$\\$d=\mathbf{-1.23}$} & \makecell*{$\mathbf{5.4\times10^{-5}}$\\$d=\mathbf{-1.62}$} & \makecell*{$\mathbf{2.2\times10^{-7}}$\\$\delta=\mathbf{-0.96}$} & \makecell*{$\mathbf{6.8\times10^{-8}}$\\$\delta=\mathbf{-1.00}$} \\
\makecell*{Titanic} & \makecell*{$\mathbf{4.2\times10^{-5}}$\\$\delta=\mathbf{-0.76}$} & \makecell*{$\mathbf{2.1\times10^{-2}}$\\$\delta=\mathbf{0.42}$} & \makecell*{$\mathbf{1.3\times10^{-2}}$\\$d=\mathbf{0.86}$} & \makecell*{$5.3\times10^{-1}$\\$d=0.22$} & \makecell*{$\mathbf{4.4\times10^{-3}}$\\$d=\mathbf{-1.04}$} & \makecell*{$\mathbf{3.5\times10^{-6}}$\\$\delta=\mathbf{-0.86}$} & \makecell*{$\mathbf{2.7\times10^{-6}}$\\$\delta=\mathbf{-0.87}$} \\
\makecell*{Diabetes} & \makecell*{$\mathbf{6.8\times10^{-8}}$\\$\delta=\mathbf{1.00}$} & \makecell*{$\mathbf{4.7\times10^{-2}}$\\$\delta=\mathbf{0.37}$} & \makecell*{$\mathbf{7.3\times10^{-5}}$\\$d=\mathbf{1.50}$} & \makecell*{$\mathbf{2.1\times10^{-3}}$\\$d=\mathbf{-1.16}$} & \makecell*{$\mathbf{4.9\times10^{-7}}$\\$d=\mathbf{-2.34}$} & \makecell*{$\mathbf{1.0\times10^{-6}}$\\$\delta=\mathbf{-0.90}$} & \makecell*{$\mathbf{1.0\times10^{-6}}$\\$\delta=\mathbf{-0.90}$} \\
\makecell*{Abalone} & \makecell*{$\mathbf{6.0\times10^{-9}}$\\$d=\mathbf{3.02}$} & \makecell*{$7.0\times10^{-1}$\\$d=-0.14$} & \makecell*{$\mathbf{1.5\times10^{-3}}$\\$d=\mathbf{1.14}$} & \makecell*{$\mathbf{6.5\times10^{-5}}$\\$d=\mathbf{-1.51}$} & \makecell*{$\mathbf{4.5\times10^{-9}}$\\$d=\mathbf{-2.69}$} & \makecell*{$\mathbf{3.9\times10^{-11}}$\\$d=\mathbf{-3.78}$} & \makecell*{$\mathbf{5.8\times10^{-11}}$\\$d=\mathbf{-3.78}$} \\
\makecell*{CaliforniaHousing} & \makecell*{$\mathbf{2.3\times10^{-5}}$\\$\delta=\mathbf{0.79}$} & \makecell*{$4.4\times10^{-1}$\\$\delta=-0.14$} & \makecell*{$\mathbf{9.0\times10^{-3}}$\\$\delta=\mathbf{0.49}$} & \makecell*{$7.4\times10^{-1}$\\$d=-0.11$} & \makecell*{$\mathbf{1.1\times10^{-2}}$\\$d=\mathbf{-0.90}$} & \makecell*{$\mathbf{1.1\times10^{-7}}$\\$\delta=\mathbf{-0.99}$} & \makecell*{$\mathbf{6.8\times10^{-8}}$\\$\delta=\mathbf{-1.00}$} \\
\makecell*{GeneExpressionCancer \\ RNA-Seq} & \makecell*{$\mathbf{1.1\times10^{-5}}$\\$\delta=\mathbf{-0.72}$} & \makecell*{$3.0\times10^{-1}$\\$\delta=0.18$} & \makecell*{$\mathbf{1.6\times10^{-3}}$\\$d=\mathbf{-1.18}$} & \makecell*{$\mathbf{3.6\times10^{-3}}$\\$d=\mathbf{-1.07}$} & \makecell*{$9.4\times10^{-1}$\\$d=0.01$} & \makecell*{$\mathbf{1.1\times10^{-5}}$\\$d=\mathbf{-1.70}$} & \makecell*{$\mathbf{2.2\times10^{-7}}$\\$\delta=\mathbf{-0.96}$} \\
\makecell*{Arcene} & \makecell*{$\mathbf{1.4\times10^{-7}}$\\$\delta=\mathbf{-0.97}$} & \makecell*{$\mathbf{1.3\times10^{-2}}$\\$d=\mathbf{-0.88}$} & \makecell*{$\mathbf{2.4\times10^{-8}}$\\$d=\mathbf{2.36}$} & \makecell*{$\mathbf{5.6\times10^{-4}}$\\$d=\mathbf{-1.24}$} & \makecell*{$\mathbf{2.5\times10^{-6}}$\\$d=\mathbf{-2.05}$} & \makecell*{$\mathbf{2.2\times10^{-7}}$\\$\delta=\mathbf{-0.97}$} & \makecell*{$\mathbf{1.0\times10^{-7}}$\\$\delta=\mathbf{-1.00}$} \\
\makecell*{Dorothea} & \makecell*{$\mathbf{7.5\times10^{-13}}$\\$d=\mathbf{-4.01}$} & \makecell*{$\mathbf{4.2\times10^{-8}}$\\$d=\mathbf{-2.25}$} & \makecell*{$\mathbf{4.0\times10^{-7}}$\\$d=\mathbf{2.13}$} & \makecell*{$\mathbf{2.8\times10^{-3}}$\\$\delta=\mathbf{-0.56}$} & \makecell*{$\mathbf{5.9\times10^{-6}}$\\$\delta=\mathbf{-0.84}$} & \makecell*{$\mathbf{2.6\times10^{-7}}$\\$\delta=\mathbf{-0.96}$} & \makecell*{$\mathbf{3.4\times10^{-7}}$\\$\delta=\mathbf{-0.95}$} \\
\bottomrule
\end{tabular}

\end{table*}

Fig.~\ref{fig:deltas} illustrates the sensitivity of the best achieved OOB performance to the tolerance parameter $\varepsilon$ on three representative datasets.
Fig.~\ref{fig:bars} summarizes the wall-clock tuning time and the selected number of trees $T$ for TPE, HB, ES, and PLATEAU across all datasets.

The option \param{only\_depth} is used to isolate the effect of tuning tree depth.
When \param{only\_depth}=``Yes'', the following hyperparameters are fixed at their default values:
\param{max\_features}=$\sqrt{p}$, \param{min\_samples\_split}=2, and \param{min\_samples\_leaf}=1,
so that only \param{max\_depth} is tuned over a predefined range \param{max\_depth\_range}.
When \param{only\_depth}=``No'', the above hyperparameters are also included in the search space.
All remaining values and ranges were kept at their defaults, as specified above in the ``Implementation details'' subsection.

We deliberately study the roles of \param{tune\_criterion} and \param{only\_depth} in hyperparameter tuning, because in practical workflows analysts often fix the split criterion and keep \param{max\_features}, \param{min\_samples\_split}, and \param{min\_samples\_leaf} at their default values.
Indeed, under a conservative tuning budget one can restrict the search to the tree depth and the number of trees: common heuristics such as \param{max\_features}=$\sqrt{p}$ (or $p/3$) are widely used for Random Forest, whereas \param{min\_samples\_leaf} and \param{min\_samples\_split} are to a large extent coupled with the depth constraint \param{max\_depth}.
We evaluate the effects of \param{tune\_criterion} and \param{only\_depth} for both the PLATEAU search and the classic TPE baseline, which also provides a useful standalone comparison of tuning protocols.

Next, we compare the resulting Random Forest OOB score across several controlled contrasts.
First, we study the effect of the trial budget by comparing $n_{\text{trials}}=120$ vs.\ $40$ at fixed dataset, \param{tune\_criterion}, \param{only\_depth}, and algorithm (Table~\ref{tab:n_trials_120_vs_40}).
Second, at $n_{\text{trials}}=120$ we compare ``Yes'' vs.\ ``No'' for \param{tune\_criterion} and, separately, for \param{only\_depth}, at fixed dataset and algorithm (Table~\ref{tab:criterion_depth_yes_vs_no}).
Third, we report additional ablations and sensitivity checks (Table~\ref{tab:joint_hb_sf}) that complement the main score-based comparisons by considering (i) joint vs.\ decoupled tuning of the tree budget and the remaining hyperparameters (for both TPE and PLATEAU), (ii) pruning behavior (HB vs.\ PLATEAU), (iii) computational cost comparisons (HB vs.\ TPE, in wall-clock time and total number of trees built), and (iv) the sensitivity of the PLATEAU procedure to the scale factor $\textnormal{sf}$ under otherwise fixed settings.
Each contrast compares two groups of 20 runs corresponding to different random seeds.

All reported score comparisons are therefore OOB-based tuning comparisons rather than independent held-out test-set evaluations.
This design is intentional: the goal of the study is to compare tree-count selection mechanisms under the same OOB-driven HPO objective.
As in any HPO procedure that selects the best configuration among many noisy validation estimates, the best observed OOB score may be affected by selection-induced upward bias.
This effect is not specific to PLATEAU and also applies to the OOB-based TPE, HB, and ES baselines considered here.

For each two-group comparison, we report both statistical significance (via a $p$-value) and an effect-size estimate.
We first assess normality within each group using the Shapiro--Wilk test.
If normality is supported, we use a two-sample $t$-test together with Cohen's $d$; otherwise, we use the Mann--Whitney $U$-test together with Cliff's $\delta$
(Field~\cite{Field2024Discovering}).
Effect sizes are additionally averaged over bootstrap resamples to stabilize the estimates under the small group size of $20$ runs.
We declare differences significant at the conventional level 
$p$-value $< 0.05$.
Following common guidelines, we treat effect sizes as practically negligible when Cohen's $d<0.5$ or Cliff's $\delta<0.28$.
Statistically notable entries are highlighted in bold in all tables.
Across the three tables, the $p$-values and effect sizes are typically consistent (both crossing or both staying within their respective thresholds), which supports the adequacy of the chosen testing protocol for our design.

In Table~\ref{tab:n_trials_120_vs_40}, we use a one-sided test for the $p$-value, since a larger trial budget can reasonably be expected to improve the best-achieved model.
For metrics where larger values indicate better performance (accuracy and ROC-AUC), this corresponds to testing whether the score is higher for $n_{\text{trials}}=120$.
For RMSE, where smaller values are better, the direction is reversed accordingly.
In Tables~\ref{tab:criterion_depth_yes_vs_no} and \ref{tab:joint_hb_sf}, we use two-sided tests.
A useful property of both Cohen's $d$ and Cliff's $\delta$ is that they are signed, indicating the direction of the effect.
The effect-size sign follows the sign of the difference between the group means (or locations): 
for ``120 vs.\ 40'' it corresponds to $\text{Score}_{120}-\text{Score}_{40}$, and for ``Yes vs.\ No'' to $\text{Score}_{\text{Yes}}-\text{Score}_{\text{No}}$.
As a result, for regression tasks a strong improvement can correspond to a negative sign, since lower RMSE implies better performance.
\subsection{Insights and Discussion}
As Fig.~\ref{fig:deltas} shows, smaller values of $\varepsilon$ impose a stricter plateau condition and therefore tend to push the search toward larger ensemble sizes, i.e., deeper into the plateau region (Fig.~\ref{fig:plateau}). Beyond the best achieved OOB metric, the figure also shows that decreasing the tolerance $\varepsilon$ is accompanied by a smaller run-to-run spread across random seeds (narrower boxes and shorter whiskers), that is, by higher reproducibility stability of the resulting OOB performance, as discussed above.

Fig.~\ref{fig:bars} suggests that, for most datasets, the proposed PLATEAU search with $\varepsilon=10^{-3}$ often achieves shorter tuning times than the fixed-range baselines TPE and HB with $T_{\max}=2565$. 
This pattern also holds for the high-dimensional \dataset{GeneExpressionCancer RNA-Seq} dataset, whereas notable exceptions include \textsc{Diabetes}, \dataset{Arcene}, and \dataset{Dorothea}. 
We stress, however, that this comparison is mainly illustrative rather than a strict ablation, since TPE and HB are controlled by $T_{\max}$, whereas PLATEAU is controlled by the tolerance $\varepsilon$. 
Accordingly, increasing $T_{\max}$ would make TPE and HB more expensive, whereas decreasing $\varepsilon$ would increase the cost of PLATEAU by forcing the search deeper into the plateau region.

HB is usually somewhat faster than TPE, although the effect is not large. A plausible explanation is that when the score grows only weakly over much of $[T_{\min},T_{\max}]$, TPE often samples ensemble sizes noticeably to the left of the right boundary, which partly offsets the advantage of Hyperband-style pruning. 
By contrast, HB attempts to traverse the ladder toward $T_{\max}$ within each trial, but many such traversals are interrupted earlier by pruning. 
As Fig.~\ref{fig:bars} indicates, the contrast between PLATEAU and these $T_{\max}$-guided baselines is stronger than the contrast between TPE and HB themselves.

A more direct ablation is given by PLATEAU versus ES, since both methods are tolerance-guided and do not rely on an explicit $T_{\max}$. 
As expected, PLATEAU typically selects larger ensemble sizes and therefore often incurs higher computational cost than ES. 
This is consistent with the downward bias of naive monotone early stopping, which tends to underestimate the sufficient number of trees because the stopping condition is checked sequentially from smaller to larger ensembles. 
The main exceptions are datasets for which PLATEAU performs left shifts and ultimately selects $T<T_{\min}=100$.

Finally, the high-dimensional datasets \dataset{Arcene} and especially \dataset{Dorothea} illustrate the main scenario in which an adaptive tree-count mechanism is useful.
For such problems, the OOB score may stabilize slowly, and rule-of-thumb values such as $T_{\max}\approx 2000$ or even $5000$ can be insufficient.
In this regime, PLATEAU tends to build substantially larger forests because it does not impose a fixed resource ladder and continues increasing the ensemble size until the relative right-side OOB change falls below the prescribed tolerance.
This behavior explains the exceptionally large selected tree counts and runtimes observed for \dataset{Dorothea}: the method treats the lack of an OOB plateau as evidence that the current ensemble size is still insufficient, rather than as a reason to stop early.
Thus, PLATEAU is expected to be most beneficial when the usual tree-count heuristics or fixed Hyperband ladder are too restrictive, whereas it may be computationally less attractive than HB or ES on datasets where the OOB curve stabilizes quickly or where an approximate, fixed-budget solution is sufficient.

Table~\ref{tab:n_trials_120_vs_40} compares the best achieved OOB score for
$n_{\text{trials}}=120$ against $n_{\text{trials}}=40$. 
Overall, increasing the
trial budget leads to statistically significant score changes in most cells,
which is expected because a larger number of trials gives the sampler more
opportunities to explore the hyperparameter space. 
Importantly, this pattern is
observed for both TPE and the proposed PLATEAU procedure. 
Thus, with respect to
the effect of the trial budget, PLATEAU behaves similarly to the standard TPE
baseline and does not reveal additional artifacts caused by the internal
tree-count adaptation.

A consistent trend in Table~\ref{tab:n_trials_120_vs_40} is that, at fixed
algorithm and fixed \param{tune\_criterion}, the columns with
\param{only\_depth} = ``No'' tend to contain smaller $p$-values than their
\param{only\_depth} = ``Yes'' counterparts. 
Some of the significant cells in
the restricted setting lie only in the borderline range $0.01$--$0.05$,
especially when only the tree depth is tuned. 
This is expected: when more
hyperparameters are tuned jointly, a larger trial budget is required for the
TPE sampler to explore the higher-dimensional space, whereas in the restricted
setting (\param{only\_depth} = ``Yes'') convergence typically occurs faster even
with fewer trials.

In contrast, a cell-wise comparison of
\param{tune\_criterion} = ``Yes'' and \param{tune\_criterion} = ``No'' in
Table~\ref{tab:n_trials_120_vs_40}, within each algorithm and at fixed
\param{only\_depth}, reveals no systematic advantage in either direction.
This suggests that enabling criterion tuning does not, by itself, demand a
substantially larger number of trials.
Taken together, these observations indicate that the effect of increasing
$n_{\text{trials}}$ is mainly governed by the effective dimensionality of the
tuned hyperparameter space, 
and that the proposed PLATEAU mechanism follows the
same qualitative pattern as TPE rather than introducing a qualitatively
different dependence on the trial budget.

Table~\ref{tab:criterion_depth_yes_vs_no} indicates that enabling split-criterion
tuning (\param{tune\_criterion}=``Yes'' vs.\ ``No'') rarely yields a statistically
significant improvement in the final score (left block of the table).
When statistically significant differences do appear, their magnitude is highly
dataset-dependent: in many settings the effect is small or negligible, but in
some cases (e.g., \dataset{CreditCardDefault}) criterion tuning can yield a
substantial gain.
Overall, this suggests that split-criterion tuning is an optional,
dataset-specific refinement rather than a consistently beneficial knob under a
fixed trial budget.

The right block of Table~\ref{tab:criterion_depth_yes_vs_no} compares depth-only
tuning with the broader search space that also includes
\param{max\_features}, \param{min\_samples\_split}, and
\param{min\_samples\_leaf}. 
In many datasets, expanding the tuned
hyperparameter set leads to statistically significant score changes, confirming
that these parameters can materially affect the bias--variance trade-off beyond
\param{max\_depth} alone. 
However, the sign and magnitude of the effect remain
strongly dataset-dependent. 
In several cases, including \dataset{Wine},
\dataset{BreastCancer}, \dataset{Arcene}, and \dataset{Dorothea}, the broader
search space does not uniformly improve the best score under the fixed budget
$n_{\text{trials}}=120$. 
This suggests that the benefit of enlarging the search
space can be offset by the additional optimization difficulty introduced by
higher-dimensional HPO.

Most importantly for the present study, the qualitative patterns in both blocks
of Table~\ref{tab:criterion_depth_yes_vs_no} are largely synchronized between
TPE and PLATEAU. 
Split-criterion tuning remains dataset-specific for both
methods, and the effect of moving from depth-only tuning to the broader search
space changes sign on essentially the same datasets. 
Thus, the observed
behavior is better interpreted as a property of the dataset and the HPO search
space under a fixed trial budget, rather than as a PLATEAU-specific artifact.

Table~\ref{tab:joint_hb_sf} complements the main comparisons by providing additional ablations and sensitivity checks.
The first two columns quantify the benefit of joint hyperparameter optimization over the corresponding decoupled (two-stage) strategies for TPE and PLATEAU.
Overall, these columns confirm that joint HPO does matter: tuning the remaining Random Forest hyperparameters at $T_{\min}$
and only then increasing or adapting the tree count is generally not equivalent to optimizing the tree budget and the remaining hyperparameters jointly.

For the TPE baseline, the first column shows that the decoupled TPE\_Tmin-Tmax strategy is often significantly worse than the fully joint TPE search, even though its second stage evaluates the selected configuration at $T_{\max}$.
This indicates that simply growing the forest after tuning at $T_{\min}$
does not compensate for the mismatch between the hyperparameters preferred by small and large ensembles.
For PLATEAU, the corresponding contrast is less uniform but still informative.
Several datasets show borderline or statistically significant differences between TPE\_Tmin-PLT and the fully joint PLATEAU procedure, indicating that joint tuning remains relevant.
A plausible explanation for the weaker separation in some cases is that the plateau-based second stage already adapts $T$
to the selected non-$T$ hyperparameters and therefore partially mitigates the decoupling effect.
Moreover, differences are expected to be less pronounced on datasets where the selected tree count remains close to the initial value $T_0$
(see Fig.~\ref{fig:bars}), so that using a small tree budget $T=T_0$
during the first stage is already relatively adequate.
In contrast, for datasets where the selected tree count is substantially larger or the hyperparameter space is harder to explore, as in \dataset{Dorothea}, the advantage of the fully joint protocol becomes more visible; this is precisely the regime targeted by the proposed plateau-based procedure, where the tree budget and the remaining Random Forest hyperparameters are optimized jointly rather than separated into two stages.

The third column of Table~\ref{tab:joint_hb_sf} shows that, under the standard choice $\eta=3$
and the ladder capped at $T_{\max}=2565$, Hyperband pruning is typically more aggressive than the plateau-specific pruning mechanism, i.e., it yields a larger fraction of pruned trials.
The fourth and fifth columns contextualize the computational cost of HB relative to TPE.
HB and TPE exhibit comparable wall-clock tuning times, whereas the total number of trees built over the entire tuning process tends to be smaller for HB, reflecting the role of the fixed ladder and rung-wise pruning.

Finally, the last two columns report the sensitivity of PLATEAU to the scale factor.
As expected, increasing $\textnormal{sf}$ from $1.5$ to $2.0$ generally increases computational cost, because a coarser geometric progression leads to larger candidate ensembles and hence larger right endpoints $R$ in the triplet, making individual trials more expensive.

Overall, Tables~\ref{tab:n_trials_120_vs_40}, \ref{tab:criterion_depth_yes_vs_no}, and~\ref{tab:joint_hb_sf} show that the main ablation effects are interpretable and largely consistent across the TPE and PLATEAU searches.
Increasing $n_{\text{trials}}$, enabling additional hyperparameter tuning in datasets where it is useful, relaxing the depth-only restriction in suitable datasets, and using joint hyperparameter optimization rather than a decoupled strategy generally affect performance in the expected direction, although the magnitude and even the sign of some effects remain dataset-dependent.
Most importantly, the qualitative impact of these tuning settings is broadly synchronized between TPE and PLATEAU, suggesting that the proposed method preserves the baseline HPO behavior under these ablations and does not introduce additional artifacts.

\section{Conclusion}
\label{sec:conclusion}
This work presents a method that fundamentally addresses the tuning of the previously ``untunable'' key parameter---the number of trees in a Random Forest---by jointly optimizing it with the remaining hyperparameters via an intuitive, relative-error-based plateau criterion.
The proposed triplet-based plateau-search algorithm, implemented in the accompanying \texttt{rf\_plateau\_hpo} library, automates this process and removes the need to predefine an arbitrary search range $[T_{\min},T_{\max}]$, replacing it with the more interpretable tolerance parameter $\varepsilon$.
At the same time, it mitigates the underestimation and unreliability of tree-count estimates that are typical of naive early-stopping schemes, without incurring the additional cost of multiple-forest training. 
Instead, it exploits the natural variability accumulated across HPO trials together with the across-trial dependence of the adaptive tree-count updates.

Our experiments yield the expected qualitative trends: increasing the trial budget and expanding the set of tuned hyperparameters generally improves the final model quality, and the plateau-based approach exhibits behavior consistent with the classic TPE baseline under these ablations.
At the same time, we emphasize that selecting $\varepsilon$ is a responsible modeling choice, especially for classification metrics, where it should not be set below the natural resolution of the empirical score.

The paper also contributes a theoretical interpretation of the proposed criterion. 
Under a power-law tail model for the score trajectory and a variance--covariance scaling motivated by earlier studies of finite-ensemble fluctuations, 
we linked the observed relative OOB-score difference to the gap between the current forest and the limiting score, and derived asymptotic expressions for the conditional variance of both the signed and absolute relative score differences. 
The resulting variance decay with ensemble size is consistent with the empirically observed bounded fluctuation regime of the adaptive tree-count updates across trials, rather than with unbounded growth toward ever larger forests. 
A more explicit probabilistic analysis of this across-trial dynamics remains a subject for future work.

Our algorithm serves as a critical first step before training the large-scale ensembles required for stable Variable Importance Measures in high-dimensional settings.
We posit that this makes the method particularly valuable for domains such as bioinformatics.
Consequently, a detailed benchmark on large-scale genomic and proteomic datasets represents a natural direction for future work.
Expanding the library and rigorously evaluating it in this context are the focus of a forthcoming in-depth study.

\bibliographystyle{IEEEtran}
\bibliography{refs, refs_stability}

\appendices
\section*{Appendix}
\begin{proof}[Proof of Proposition~\ref{prop:plateau_limit_alpha}]
From \eqref{eq:score_tail_alpha},
\begin{equation}
S_B-S_\infty = cB^{-\gamma}+o(B^{-\gamma}),
\label{eq:SB_limit_tail}
\end{equation}
and, since $R=\textnormal{sf}\cdot B$,
we similarly have $S_R-S_\infty = c\,\textnormal{sf}^{-\gamma}B^{-\gamma}+o(B^{-\gamma})$. 
Therefore,
\begin{equation}
S_B-S_R
=
c(1-\textnormal{sf}^{-\gamma})B^{-\gamma}+o(B^{-\gamma}).
\label{eq:SB_SR_tail}
\end{equation}
Because $c(1-\textnormal{sf}
^{-\gamma})\neq 0$, dividing \eqref{eq:SB_limit_tail} by
\eqref{eq:SB_SR_tail} yields \eqref{eq:ratio_limit_plateau}.
By reformulating \eqref{eq:ratio_limit_plateau} in the 
equivalent form
\begin{equation*}
|S_\infty-S_B|
=
\left(\frac{1}{1-\textnormal{sf}^{-\gamma}}+o(1)\right)|S_R-S_B|,
\qquad B\to\infty,
\end{equation*}
dividing by $|S_B|$ and applying the inequality \eqref{eq:plateau_condition_BR}, we obtain \eqref{eq:plateau_limit_alpha_rel}.
\end{proof}

\begin{proof}[Proof of Lemma~\ref{lemma:delta}]
Define
\begin{equation*}
g(x,y)=\frac{y-x}{x}=\frac{y}{x}-1,
\end{equation*}
so that $(S_R-S_B)/S_B=g(S_B,S_R)$.
Since $\mu_B\neq 0$, the function $g$ is differentiable in a neighborhood of $(\mu_B,\mu_R)$, and its second-order Taylor expansion takes the form
\begin{multline*}
g(S_B,S_R)\approx g(\mu_B,\mu_R)
+\nabla g(\mu_B,\mu_R)^\top
\begin{pmatrix}
S_B-\mu_B\\
S_R-\mu_R
\end{pmatrix} \\
+\frac12
\begin{pmatrix}
S_B-\mu_B\\
S_R-\mu_R
\end{pmatrix}^{\!\top}
H_g(\mu_B,\mu_R)
\begin{pmatrix}
S_B-\mu_B\\
S_R-\mu_R
\end{pmatrix},
\end{multline*}
where
\begin{equation*}
\nabla g(x,y) =
\begin{pmatrix}
-y/x^2 \\
1/x
\end{pmatrix}, \quad
H_g(x,y)=
\begin{pmatrix}
2y/x^3 & -1/x^2\\
-1/x^2 & 0
\end{pmatrix}.
\end{equation*}
Hence the zeroth-order term and the gradient are
\begin{equation*}
g(\mu_B,\mu_R) = \frac{\mu_R-\mu_B}{\mu_B}, \quad
\nabla g(\mu_B,\mu_R)=
\begin{pmatrix}
-\mu_R/\mu_B^2 \\ 1/\mu_B
\end{pmatrix},
\end{equation*}
whereas the second-order term can be written explicitly as
\begin{multline*}
\frac12
\begin{pmatrix}
S_B-\mu_B\\
S_R-\mu_R
\end{pmatrix}^{\!\top}
H_g(\mu_B,\mu_R)
\begin{pmatrix}
S_B-\mu_B\\
S_R-\mu_R
\end{pmatrix} \\
=
\frac{\mu_R}{\mu_B^3}(S_B-\mu_B)^2
-\frac{1}{\mu_B^2}(S_B-\mu_B)(S_R-\mu_R).
\end{multline*}

Taking the conditional expectation given $D$, 
only zeroth- and second-order terms contribute, whereas the linear term vanishes. 
So we obtain
\begin{equation*}
\mathbb{E}\left[\frac{S_R-S_B}{S_B}\middle|D\right]
\approx
\frac{\mu_R-\mu_B}{\mu_B}
+\frac{\mu_R}{\mu_B^3}\sigma_B^2
-\frac{1}{\mu_B^2}\sigma_{BR}.
\end{equation*}

Under the approximate bivariate normal representation of $(S_B,S_R)\mid D$, the first-order delta-method approximation gives (van der Vaart~\cite{Vaart1998Asymptotic})
\begin{IEEEeqnarray*}{rCl}
\IEEEeqnarraymulticol{3}{l}{\mathrm{Var}\left[\frac{S_R-S_B}{S_B}\;\bigg|\;D\right]} \\
&\qquad\approx&
\nabla g(\mu_B,\mu_R)^\top
\begin{pmatrix}
\sigma_B^2 & \sigma_{BR}\\
\sigma_{BR} & \sigma_R^2
\end{pmatrix}
\nabla g(\mu_B,\mu_R) \\
&\qquad=&
\begin{pmatrix}
-\dfrac{\mu_R}{\mu_B^2} & \dfrac{1}{\mu_B}
\end{pmatrix}
\begin{pmatrix}
\sigma_B^2 & \sigma_{BR}\\
\sigma_{BR} & \sigma_R^2
\end{pmatrix}
\begin{pmatrix}
-{\mu_R}/{\mu_B^2}\\[3mm]
{1}/{\mu_B}
\end{pmatrix} \\
& \qquad=&
\frac{\mu_R^2}{\mu_B^4}\sigma_B^2
+
\frac{1}{\mu_B^2}\sigma_R^2
-
2\frac{\mu_R}{\mu_B^3}\sigma_{BR}.
\end{IEEEeqnarray*}
\end{proof}

\begin{proof}[Proof of Proposition~\ref{prop:absolute_variance}]
With $R=\textnormal{sf}\cdot B$, $B\to\infty$, the tail condition for $\mu_T$ gives
$\mu_B=S_\infty+cB^{-\gamma}+o(B^{-\gamma})$ and
$\mu_R=S_\infty+c\,\textnormal{sf}^{-\gamma}B^{-\gamma}+o(B^{-\gamma})$, hence
$\left(\mu_R-\mu_B\right)/\mu_B=O(B^{-\gamma})$.
At the same time, by \eqref{eq:remark} and $\mu_B,\mu_R\to S_\infty$, 
the second-order term in the mean approximation \eqref{eq:delta_mean_G} given in Lemma~\ref{lemma:delta} is

\begin{equation*}
\frac{\mu_R}{\mu_B^3}\sigma_B^2
-\frac{1}{\mu_B^2}\sigma_{BR} =
\frac{v(1-\textnormal{sf}^{-1})}{S_{\infty}^2}B^{-1} + o(B^{-1}).
\end{equation*}
Therefore,
\begin{equation*}
\mathbb{E}\left[\frac{S_R-S_B}{S_B}\middle|D\right]
=
O(B^{-\gamma})+O(B^{-1})
=
O\!\left(B^{-\min(\gamma,1)}\right),
\end{equation*}
Moreover, the Taylor terms of order higher than two contribute only $o(B^{-1})$ to the conditional mean and therefore do not affect the above estimate.

On the other hand, Proposition~\ref{prop:delta_asymptotic} gives
\begin{equation*}
\mathrm{Var}\left[\frac{S_R-S_B}{S_B}\middle|D\right]
=
\frac{v(1-\textnormal{sf}^{-1})}{S_\infty^2}B^{-1}+o(B^{-1}).
\end{equation*}
Hence, if we denote
\begin{equation*}
\rho_B=
\frac{
\mathbb{E}\left[(S_R-S_B)/S_B\middle|D\right]
}{
\sqrt{
\mathrm{Var}\left[(S_R-S_B)/S_B\middle|D\right]
}
},
\end{equation*}
then
\begin{equation}
\rho_B=O\!\left(B^{1/2-\min(\gamma,1)}\right)\to 0,
\qquad B\to\infty,
\label{eq:mean_shift_vanishes}
\end{equation}
since $\gamma>1/2$.

Because $(S_R-S_B)/S_B$, conditionally on $D$, is approximately Gaussian, the absolute value
$\left|(S_R-S_B)/S_B\right|$ is approximately folded normally distributed, with variance
\begin{multline*}
\mathrm{Var}\left[\left|\frac{S_R-S_B}{S_B}\right|\middle|D\right]
\approx
\mathrm{Var}\left[\frac{S_R-S_B}{S_B}\middle|D\right] \\
\times
\Bigg(
1+\rho_B^2-
\bigg[
\sqrt{\frac{2}{\pi}}e^{-\rho_B^2/2}
+
\rho_B\bigl(1-2\Phi(-\rho_B)\bigr)
\bigg]^2
\Bigg).
\end{multline*}
Using \eqref{eq:mean_shift_vanishes} and passing to the limit, we obtain
\begin{equation*}
\mathrm{Var}\left[\left|\frac{S_R-S_B}{S_B}\right|\middle|D\right]
\approx
\left(1-\frac{2}{\pi}\right)
\mathrm{Var}\left[\frac{S_R-S_B}{S_B}\middle|D\right].
\end{equation*}
Substituting the asymptotic expression from Proposition~\ref{prop:delta_asymptotic} yields
\eqref{eq:prop3_main}.
\end{proof}

\begin{IEEEbiography}[{\includegraphics[width=1in,height=1.25in,clip,keepaspectratio]{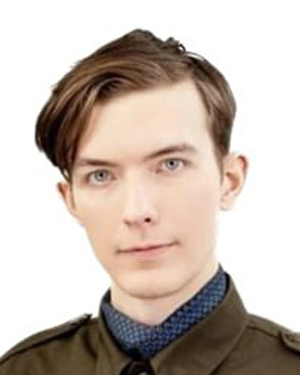}}]{Vadim Porvatov} is an M.S. graduate of the University of Amsterdam with expertise in graph machine learning, deep learning, and computer vision. His research experience includes development of applied machine learning models for real-world prediction tasks in temporal and graph-based domains.
\end{IEEEbiography}

\begin{IEEEbiography}[{\includegraphics[width=1in,height=1.25in,clip,keepaspectratio]{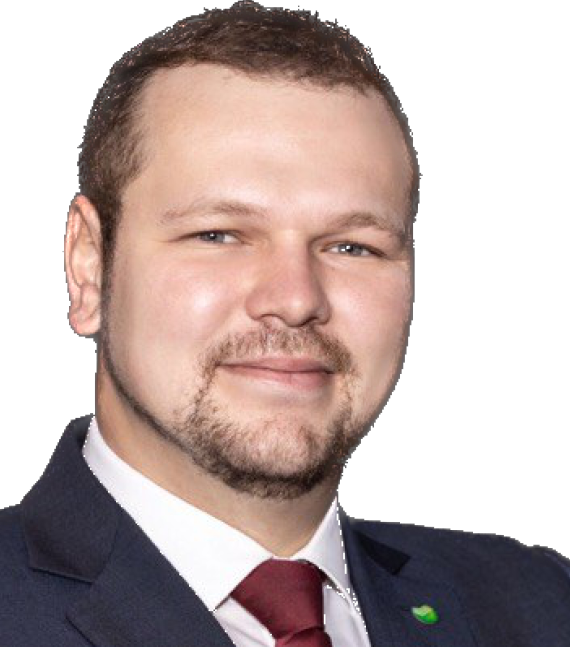}}]
{Andrey Dukhovny} graduated from the Faculty of Aerophysics and Space Research at the Moscow Institute of Physics and Technology and earned an M.S. in Economics from RANEPA in 2013. His key professional roles include Chief Data Scientist at Sberbank, where he led AI-driven transformation initiatives, and Deputy Head of Liquidity Risk Management with expertise in ALM.
\end{IEEEbiography}

\begin{IEEEbiography}[{\includegraphics[width=1in,height=1.25in,clip,keepaspectratio]{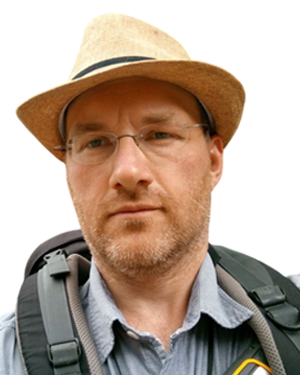}}]{Andrey Lange} received the M.Sc. degree in applied mathematics from the Bauman Moscow State Technical University (BMSTU), 
in 2002, and the Ph.D. degree in probability theory and statistics from BMSTU, in 2007. 
He is currently working 
at the Skolkovo Institute of Science and Technology (Skoltech) as Research Group Leader
and
at the Federal Research Center ``Computer Science and Control" of RAS as a Research Fellow. 
His research interests include stochastic processes, interpretable models in machine learning, and its applications.
\end{IEEEbiography}
\EOD
\end{document}